%% file: neurips_2026.tex
\documentclass{article}

\usepackage[eandd,final]{neurips_2026}

\usepackage[utf8]{inputenc} 
\usepackage[T1]{fontenc}    
\usepackage{hyperref}       
\usepackage{url}            
\usepackage{booktabs}       
\usepackage{amsfonts}       
\usepackage{nicefrac}       
\usepackage{microtype}      
\usepackage{xcolor}         
\usepackage{microtype}
\usepackage{graphicx}
\usepackage{subfig}
\usepackage{makecell}

\usepackage{amsmath,amsfonts,bm,amssymb,dsfont,nicefrac}
\usepackage{xspace}
\usepackage[capitalize,noabbrev]{cleveref}

\input{math_commands}

\usepackage{amsthm}
\theoremstyle{plain}
\newtheorem{theorem}{Theorem}[section]
\newtheorem{proposition}[theorem]{Proposition}

\theoremstyle{definition}

\theoremstyle{remark}

\title{\bench: A Realistic Benchmark for Time Series Forecasting}

\author{%
  Oleksandr Shchur $^{1}$ \thanks{Equal contribution} \\
  \And
  Abdul Fatir Ansari $^{1}$ \footnotemark[1] \\
  \And
  Caner Turkmen $^2$ \thanks{Work done while at Amazon}\\
  \And
  Lorenzo Stella $^{1}$\\
  \And
  Nick Erickson $^{3}$ \footnotemark[2]\\
  \And
  Pablo Guerron $^{1,4}$ \footnotemark[2]\\
  \And
  Michael Bohlke--Schneider $^1$ \\
  \And
  Yuyang Wang $^5$ \footnotemark[2]
  \And
  $^{1}$ \normalfont{Amazon}
  $\quad$ $^2$ Keystone AI
  $\quad$ $^3$ Prior Labs
  $\quad$ $^4$ Boston College
  $\quad$ $^5$ NTT DATA AIVista
}

\begin{document}

\maketitle

\begin{abstract}
Benchmark quality is critical for meaningful evaluation and sustained progress in time series forecasting, particularly with the rise of time series foundation models (TSFMs).
Existing benchmarks often have limited domain coverage or overlook real-world settings such as tasks with covariates.
Their aggregation procedures frequently lack statistical rigor, making it unclear whether performance differences reflect true improvements or random variation.
Many benchmarks lack consistent evaluation infrastructure or are too rigid for integration into existing pipelines.
To address these gaps, we propose \bench, a benchmark of 100 forecasting tasks across seven domains, including 46 with covariates.
Supporting the benchmark, we introduce \fev, a lightweight Python library for forecasting evaluation emphasizing reproducibility and integration with existing workflows.
Using \fev, \bench employs principled aggregation with bootstrapped confidence intervals to report performance along two dimensions: win rates and skill scores.
We evaluate various models on \bench and find that existing TSFMs often miss accuracy by ignoring covariates, pointing to a clear direction for future work.

\end{abstract}

\section{Introduction}
\looseness=-1
Pretrained time series forecasting models are transforming forecasting practice.
They often outperform traditional methods \citep{aksu2024gift} while enabling zero-shot inference that simplifies production use and lowers the barrier to entry \citep{cohen2025toto}.

\looseness=-1
Advances in pretrained forecasting models are primarily evaluated on benchmarks, making benchmark quality critical for progress.
Shortcomings in existing benchmarks directly shape model development; for example, most general benchmarks ignore covariates despite their prevalence in real-world applications, and covariates have so far been evaluated only on individual datasets, such as those from forecasting competitions~\citep{bojer2021kaggle}, or narrow collections~\citep{arango2025chronosx}.
Consequently, most pretrained models lack covariate support, limiting their effectiveness in domains such as retail, where promotional and pricing data are essential for accurate demand forecasting \citep{fildes2022retail}.

Beyond task coverage, existing benchmarks often lack statistical rigor.
Most studies report single-number summaries, making it unclear whether improvements reflect true advances or random variation.
Small gains may vanish or even reverse under minor benchmark changes \citep{roque2025cherry}, undermining the reliability of conclusions about which models perform better.

\looseness=-1
Finally, benchmark infrastructure poses additional barriers to progress and reproducibility.
Many benchmarks provide only standalone datasets without evaluation code, leading to inconsistent implementations and incomparable results \citep{hewamalage2023forecast}.
When infrastructure exists, it often consists of monolithic systems that bundle models, datasets, and evaluation logic with extensive dependencies, becoming unmaintainable over time.
This rigidity prevents extension to new domains or integration into existing workflows, limiting practical utility and longevity.

To address above challenges, we make three contributions:
\begin{itemize}

\item \textbf{New benchmark.} We introduce \bench,\footnote{\href{https://huggingface.co/spaces/autogluon/fev-bench}{\texttt{huggingface.co/spaces/autogluon/fev-bench}}} a forecast evaluation benchmark containing 100 tasks spanning 7 real-world application domains, built from 96 datasets, 40 of which have rarely or never been used to evaluate pretrained models.
Unlike existing general forecasting benchmarks, \bench includes 46 tasks with covariates alongside univariate and multivariate settings. Our evaluation highlights the importance of covariates for forecasting accuracy, despite most pretrained forecasting models and benchmarks largely ignoring them.

\item \textbf{Aggregation methods.} In our benchmark, we employ principled aggregation strategies including bootstrap-based confidence intervals that quantify whether performance differences are statistically meaningful.
This approach enables more reliable model comparisons and assesses the robustness of conclusions to variations in benchmark composition.

\item \textbf{Evaluation package.} We introduce \fev,
\footnote{\href{https://github.com/autogluon/fev}{\texttt{github.com/autogluon/fev}}}
a lightweight Python package for forecasting evaluation that introduces minimal dependencies while remaining compatible with popular forecasting libraries. The package focuses on reproducibility and extensibility, enabling researchers to easily build and share new benchmarks and the corresponding results.

\end{itemize}

\begin{table}[t]
    \centering
    \resizebox{\textwidth}{!}{\input{tables/bench_stats}}
    \caption{Overview of general time series forecasting benchmarks. \bench includes more unique datasets than prior benchmarks and is the first to systematically include tasks with covariates, addressing a major gap in existing evaluation frameworks.\vspace{-5mm}}
\label{tab:bench_stats}
\end{table}

\section{Preliminaries}

\textbf{Problem definition.}
The multivariate time series forecasting problem can be formally stated as follows.
We are given a collection $\{ \vy_{n,1:T} \}_{n=1}^N$
of $N$ multivariate time series.
For $n = 1,\dots,N$ and $t = 1,\dots,T$, let
$\vy_{n,t} = (y_{n,d,t})_{d=1}^D \in \mathbb{R}^D$
denote the $D$-dimensional observation vector for series $n$ at time $t$,
where $D = 1$ corresponds to univariate forecasting.
The goal is to predict the future $H$ values
$\vy_{n,T+1:T+H}$
for each series $n$, where $H$ is the forecast horizon.
Each time series may be accompanied by covariates
\smash{$\mathbf{X}_{n,1:T+H}$},
including (i) \emph{static covariates} that do not vary over time (e.g., location),
(ii) \emph{past-only dynamic covariates} observed up to time $T$ (e.g., past related series),
and (iii) \emph{known dynamic covariates} available for all time steps $1,\dots,T+H$ (e.g., holiday indicators).

The goal of probabilistic forecasting is to model the distribution
$\smash{p\big(\vy_{n,T+1:T+H} \mid \vy_{n,1:T}, \mathbf{X}_{n,1:T+H}\big)}$.
While full distributional modeling provides the richest information, it is common in practice to produce \emph{point forecasts} such as conditional means or medians.
Alternatively, many applications estimate \emph{predictive quantiles}
$\smash{p(y_{n,d,t} \mid \vy_{n,1:T}, \mathbf{X}_{n,1:T+H})}$.
We denote by $\gQ \subset (0,1)$ the set of quantile levels and produce forecasts
$\smash{\hat{y}^{(q)}_{n,d,t}}$ satisfying
$\smash{\Pr(y_{n,d,t} \le \hat{y}^{(q)}_{n,d,t}) = q}$ for all $q \in \gQ$.

\textbf{Benchmarks and tasks.}
A benchmark consists of forecasting tasks together with an evaluation and aggregation procedure. A task specifies a dataset and the parameters defining how forecasts are produced and evaluated, including the forecast horizon $H$, evaluation cutoffs, target and covariate columns, and evaluation metric. A single dataset can yield multiple tasks by varying these parameters. Since different metrics correspond to different optimal forecasts, combining conflicting metrics within a task creates ambiguity about the intended objective \citep{kolassa2020best}.

Each task is evaluated using a rolling-origin protocol with $W$ windows \citep{hyndman2018forecasting}. Let $\tau_1 < \tau_2 < \dots < \tau_W$ denote the evaluation cutoffs. At each window $w$, the model receives observations up to $\tau_w$ and produces $H$-step forecasts, yielding $W$ forecast--target pairs per task. This setup mimics real-world deployment and improves robustness, especially for datasets with few series.

\textbf{Aggregation.}
While tasks define individual evaluation problems, benchmarks must aggregate results across tasks to answer questions such as ``Is model A more accurate than model B overall?'' The choice of aggregation directly affects the reliability and interpretability of benchmark results.

\begin{table}[t]
\centering
\begin{minipage}{0.48\columnwidth}
    \centering
    \resizebox{\linewidth}{!}{\input{tables/dataset_domains}}
    \caption{Number of datasets from different domains in GIFT-Eval, TIME and \bench.}
    \label{tab:dataset-domains}
\end{minipage}
\hfill
\begin{minipage}{0.48\columnwidth}
    \centering
    \resizebox{\linewidth}{!}{\input{tables/dataset_freqs}}
    \caption{Number of datasets with different frequencies in GIFT-Eval, TIME, and \bench.}
    \label{tab:dataset-freqs}
\end{minipage}
\end{table}

\section{Task definitions}
\label{sec:tasks}

We introduce \bench (\textbf{F}orecast \textbf{EV}aluation Benchmark), a benchmark designed to address limitations of existing forecasting evaluation frameworks. Unlike prior benchmarks with narrow domain coverage or limited evaluation infrastructure, \bench enables reproducible and statistically sound evaluation across diverse real-world forecasting applications. \bench consists of 100 forecasting tasks, with full specifications in \cref{app:tasks}. This section focuses on dataset and task design; aggregation methods and the evaluation package are described in \cref{sec:aggregation} and \cref{sec:code}.

\subsection{Datasets and tasks}
Our goal is to construct a \textit{general} benchmark representative of real-world forecasting applications across domains, frequencies, horizons, and time series characteristics.
We consider both univariate and multivariate forecasting problems, with covariates covering dynamic (both past-only and known) and static variables.
We evaluate both point and probabilistic forecasting performance.

\looseness=-1
\textbf{Datasets.}
We source datasets from established collections including the Monash repository~\citep{godahewa2021monash}, GIFT-Eval~\citep{aksu2024gift}, and BOOM~\citep{cohen2025toto}. Since these collections lack covariates, we additionally include public datasets from Kaggle~\citep{bojer2021kaggle} and domain-specific repositories~\citep{wang2023benchmarks,lago2021forecasting,arango2025chronosx}.
This process yields 96 unique datasets and 100 forecasting tasks with different target and covariate selections. Among them, 30 tasks include known dynamic covariates, 24 past dynamic covariates, and 19 static covariates. These categories are non-exclusive, and both univariate and multivariate tasks may include covariates. Tables~\ref{tab:dataset-domains} and \ref{tab:dataset-freqs} compare \bench against GIFT-Eval \citep{aksu2024gift} and TIME \citep{qiao2026time}, showing substantially broader coverage across domains and frequencies.
Like existing general benchmarks, \bench is built by curating public datasets rather than collecting new data, yet 40 of its 96 datasets have rarely or never been used to evaluate pretrained models \citep{qiao2026time}, and popular datasets such as M5 and Rossmann are released together with the covariates discarded by existing collections \citep{aksu2024gift,woo2024unified}.

\textbf{Forecast horizons.}
Many existing benchmarks reuse the same datasets with different horizons \citep{zeng2023transformers,aksu2024gift}, creating correlated tasks with limited additional insight. While such setups are useful for studying sensitivity to forecast horizon, \bench instead prioritizes diversity across datasets and application domains. We therefore avoid horizon duplication within datasets and select horizons that reflect realistic forecasting needs, such as 168 steps for hourly energy demand or 30 steps for daily retail sales.

\looseness=-1
\textbf{Rolling evaluation.}
We use rolling window evaluation to balance computational cost and statistical reliability. Depending on dataset size, we use up to 20 windows for datasets with fewer than 10 series, up to 10 for datasets with 10--2000 series, and 1 for larger datasets, additionally requiring at least $(2 \times H + 1)$ past observations before the first window. Results are averaged across windows.

\textbf{Representative subset of tasks.}
In addition to the full 100-task \bench benchmark, we provide \minibench, a curated subset of 20 tasks (\cref{app:mini-bench}). The subset preserves the diversity of covariates, dimensionalities, domains, and horizons in the full benchmark while enabling faster iteration at lower computational cost. \minibench also approximates the relative model rankings of the full benchmark, making it suitable for model development and ablation studies.

\subsection{Evaluation metrics}
Each of the 100 benchmark tasks evaluates point and probabilistic forecast accuracy using complementary metrics.
We evaluate point forecast accuracy using Mean Absolute Scaled Error (MASE), following existing benchmarks \citep{aksu2024gift,godahewa2021monash}.
\begin{align}
    \operatorname{MASE} = \frac{1}{NDH} \sum_{n=1}^{N} \sum_{d=1}^{D} \frac{1}{a_{n,d}} \sum_{t=T+1}^{T+H} |y_{n,d,t} - \hat{y}_{n,d,t}|,
\end{align}
where the error for each series $n$ and dimension $d$ is normalized by the historical seasonal error $\smash{a_{n,d} = \frac{1}{T-m} \sum_{t=m+1}^T |y_{n,d,t} - y_{n,d,t-m}|}$.
Here $\hat{y}_{n,d,t}$ is the point forecast and $m$ is the seasonal period based on the data frequency (e.g., $m{=}12$ for monthly data).
MASE offers several advantages: it is scale-free, balances contributions across series with different magnitudes, handles trends well, and remains robust when the forecast horizon contains zeros \citep{hyndman2006another}.

We use Scaled Quantile Loss (SQL) computed on levels $\gQ = \{0.1, 0.2, ..., 0.9\}$ for probabilistic forecast accuracy:
\begin{align}
    \operatorname{SQL} = \frac{1}{NDH} \sum_{n=1}^{N} \sum_{d=1}^{D} \frac{1}{a_{n,d}} \sum_{t=T+1}^{T+H} \sum_{q \in \gQ}  \rho_q(y_{n,d,t}, \hat{y}^{(q)}_{n,d,t})
\end{align}
where $\hat{y}^{(q)}_{n,d,t}$ is the quantile forecast at level $q$, $a_{n,d}$ is the historical seasonal error as in MASE, and $\rho_q(y, \hat{y}) = 2\max\{q(y-\hat{y}), (1-q)(\hat{y}-y)\}$ is the quantile loss.

\looseness=-1
We adopt Scaled Quantile Loss (SQL) as the primary probabilistic metric because it is the natural extension of MASE and inherits its scale-independence properties. Although forecasting benchmarks more commonly use the scale-dependent Weighted Quantile Loss (WQL) \citep{ansari2024chronos,aksu2024gift}, both SQL and WQL are related to CRPS \citep{gneiting2007strictly}, Winkler Score, and Weighted Interval Score \citep{tibshirani2023forecast}. The key difference is that SQL normalizes each series by scale, whereas WQL aggregates errors in a scale-dependent manner, analogous to the distinction between MASE and Weighted Absolute Percentage Error (WAPE). Our choice follows the rationale of the M4 and M5 competitions, which also used SQL-equivalent metrics \citep{makridakis2020m4,makridakis2022m5}.

Both MASE and SQL can encounter numerical issues with intermittent time series where the seasonal error $a_{n,d}$ approaches zero \citep{hewamalage2023forecast}.
We verified that this problem does not occur in our benchmark tasks.
For completeness, we also report WQL and WAPE scores in the supplement.

\section{Aggregating the results}
\label{sec:aggregation}

After evaluating $M$ models on $R$ tasks, we obtain the error matrix $E \in \mathbb{R}_{\ge 0}^{R \times M}$, where $E_{rj}$ denotes the error (e.g., MASE) of model $j$ averaged over all evaluation windows of task $r$.
Lower values correspond to more accurate forecasts.
For failed task-model pairs (e.g., due to timeouts or crashes), we substitute $E_{rj}$ with $E_{r\beta}$, where $\beta$ denotes a predefined baseline model (Seasonal Naive).

\subsection{Marginal performance}
\label{sec:aggregation-marginal}
The primary goal of any benchmark is to rank models by average performance.
We use two complementary aggregation methods that capture different aspects of model quality.

\textbf{Average win rate} $W_j$ represents the probability that model $j$ achieves lower error than another randomly chosen model $k \ne j$ on a randomly chosen task:
\begin{align}
\label{eq:avg-winrate}
    W_j = \frac{1}{R(M-1)} \sum_{r=1}^{R} \sum_{k=1, k \neq j}^{M}
       \Big(\mathds{1}(E_{rj} < E_{rk}) + \frac{1}{2} \cdot \mathds{1}(E_{rj} = E_{rk})\Big).
\end{align}
Here $\mathds{1}(\cdot)$ is the binary indicator function.
Ties ($E_{rj} = E_{rk}$) are treated as half-wins for each model. The win rate ranges from 0 (worst) to 1 (best) and provides an intuitive measure of relative model performance.
However, win rate has two limitations: it is insensitive to the magnitude of performance differences and changes as new models are added to the benchmark, motivating our second aggregation method.

\textbf{Skill score} \citep{hyndman2018forecasting} $S_j$ quantifies how much model $j$ reduces forecasting error compared to the fixed baseline model $\beta$ on average:
\begin{align}
\label{eq:avg-skillscore}
   S_j = 1 - \exp\left(\frac{1}{R}\sum_{r=1}^{R} \log \left(\clip\left(E_{rj}/E_{r\beta}; \ell, u\right) \right)\right).
\end{align}
where $\clip(x; \ell, u) = \max(\ell, \min(x, u))$ clips $x$ to the interval $[\ell, u]$.
We aggregate relative errors across tasks using geometric mean, clipping values between $\ell=10^{-2}$ and $u=100$ to avoid excessive influence from extreme values.
The skill score ranges from 1 (perfect forecasts) to $-\infty$ (arbitrarily poor performance).
Positive values indicate that the model outperforms the baseline on average, while negative values indicate underperformance.

Geometric mean aggregation is less sensitive to outliers than the arithmetic mean and ensures that the final ranking remains invariant to the choice of baseline model (as long as all relative errors stay within $[\ell, u]$) \citep{fleming1986not}.
The geometric mean appropriately handles the multiplicative nature of relative performance comparisons, averaging ratios in a meaningful way where opposing relative errors like
$\frac{1}{2}$ and $2$ cancel out.

\subsection{Pairwise comparison}
While marginal performance provides overall rankings, pairwise comparisons reveal specific model relationships that may be obscured in aggregate statistics. The above
aggregation methods can be easily generalized to comparing any two models $j$ and $k$.

\textbf{Pairwise win rate} $W_{jk}$ represents the fraction of tasks where model $j$ outperforms model $k$
\begin{align}
\label{eq:pairwise-winrate}
W_{jk} = \frac{1}{R} \sum_{r=1}^{R} \left(\mathds{1}(E_{rj} < E_{rk}) + \frac{1}{2}\cdot \mathds{1}(E_{rj} = E_{rk})\right).
\end{align}

\textbf{Pairwise skill score} $S_{jk}$ quantifies how much model $j$ reduces error compared to model $k$ on average
\begin{align}
\label{eq:pairwise-skillscore}
   S_{jk} = 1 - \exp\left(\frac{1}{R}\sum_{r=1}^{R} \log \left(\clip\left(E_{rj}/E_{rk}; \ell, u\right) \right)\right).
\end{align}

\subsection{Significance of performance differences}
A critical concern in benchmarking is the reliability of reported performance differences.
State-of-the-art claims often rest on minor improvements that may vanish under small changes to benchmark composition, casting doubt on whether they reflect genuine advances \citep{roque2025cherry}.

\looseness=-1
To address this concern, we compute 95\% confidence intervals using paired bootstrap over tasks \citep{efron1992bootstrap}.
We generate $B = 1000$ bootstrap samples by drawing rows with replacement from $E$, where each $\tilde{E}^{(b)} \in \mathbb{R}_{\ge 0}^{R \times M}$ contains $R$ tasks sampled from the original benchmark.
For each $\tilde{E}^{(b)}$, we compute the aggregate statistics to obtain bootstrap distributions such as $\{\tilde{W}_{jk}^{(b)}\}_b$.
The $(1-\alpha)$ confidence interval for the pairwise win rate $W_{jk}$ is then
\begin{align}
\left[ Q_{\nicefrac{\alpha}{2}}\!\left(\{\tilde{W}_{jk}^{(b)}\}_b\right), \; Q_{1-\nicefrac{\alpha}{2}}\!\left(\{\tilde{W}_{jk}^{(b)}\}_b\right) \right],
\end{align}
where $Q_p(\cdot)$ denotes the empirical $p$-th quantile of the bootstrap distribution.
Analogous intervals are constructed for the pairwise skill scores $S_{jk}$. These intervals quantify how conclusions about model comparisons vary under alternative benchmark compositions.

We report bootstrap confidence intervals only for the pairwise statistics ($W_{jk}, S_{jk}$), as these directly answer the question of interest: ``Does model $j$ consistently outperform model $k$ under different benchmark compositions?''
Equivalently, the intervals define paired null-hypothesis tests over tasks, with null hypotheses $W_{jk}=0.5$ (no systematic advantage) and $S_{jk}=0$ (no average error reduction).
Confidence intervals for the marginal statistics ($W_j, S_j$) instead describe the variability of each model’s average score in isolation and ignore correlations between models.

\textbf{Summary.}
Benchmark interpretation proceeds in two steps.
First, the marginal statistics ($W_j, S_j$) provide an overall model ranking.
Second, the pairwise statistics with confidence intervals ($W_{jk}, S_{jk}$) refine this picture by showing which performance differences are robust to changes in benchmark composition.
For example, if a model $j$ ranks highest by marginal win rate $W_j$ and all of its pairwise win rates $W_{jk}$ against other models $k \neq j$ have lower bounds above 50\%, then model $j$ can be regarded as outperforming every competitor with high confidence.

\subsection{Data leakage}
\label{sec:leakage}
\looseness=-1
The goal of \bench is to assess the \textit{zero-shot} capabilities of pretrained forecasting models.
Since all datasets in \bench come from public sources, parts of them may have been included in the pretraining corpora of the evaluated models, which would undermine this goal.
We distinguish two forms of leakage: (i) training on the benchmark training split and (ii) exposure to the test split.
While the second form directly inflates the reported accuracy, the first one also violates the zero-shot setting, since the model has already seen the series it is asked to forecast.

\looseness=-1
\textbf{Self-reported overlap.}
To guard against both forms of leakage, contributors must report, for each model--task pair, whether \textit{any} part of the dataset at the same frequency was used during training; resampled variants of a dataset at different frequencies are not considered leakage.
This indicator is required only for pretrained models, since task-specific models are expected to train on the task training split and the zero-shot requirement does not apply to them.
Self-reporting of training overlap is the standard practice for static benchmarks, where both historical and future values of every series are publicly available.
For example, GIFT-Eval relies on authors disclosing their pretraining data to flag in-domain results \citep{aksu2024gift}, which is exactly the information required by \bench and which is already published for every open-weight pretrained model.
Similarly, TIME introduces fresh datasets that have not been used for evaluation before \citep{qiao2026time}, but nothing prevents future models from training on them, so it ultimately depends on self-reporting as well.
One could instead attempt to detect leakage automatically, but this amounts to membership inference over pretraining corpora, which remains an open research problem even for large language models and is even harder for numerical data \citep{wahdany2026dataset}.
We therefore see self-reporting as the only practical option available to any public benchmark today, and consider automatic detection an interesting direction for future work.

\looseness=-1
\textbf{Handling leaked tasks.}
Where \bench differs from existing benchmarks is in how reported leakage is handled.
Rather than disqualifying the model or merely flagging its results, we discard the submitted results on the affected tasks and impute them with the scores of the strongest fully zero-shot univariate model at release time, Chronos-Bolt (Base), similar to the approach used in TabArena \citep{erickson2025tabarena}.
This removes the impact of leaked tasks on the aggregate scores while preserving the evaluation on the remaining tasks.
A natural alternative would be to exclude the leaked tasks altogether, but this does not hold up as \bench accepts new model submissions.
If tasks are dropped only for the affected model, each submission is scored on a different set of tasks, whereas win rates and skill scores are comparable only when computed over the same tasks.
If tasks are instead dropped for all models, the scores remain comparable, but every leaky submission permanently costs all other models a task, and the benchmark keeps shrinking as submissions accumulate.
Imputation avoids both problems: it gives contributors a reason not to train on \bench data, without making honest disclosure too costly.
In practice, this choice does not affect the conclusions of our evaluation.
When we drop the 10 tasks that are flagged as leaky for at least one model (and omit Moirai-2.0, whose reported overlap covers too many tasks to be dropped), the ranking of the remaining models stays unchanged (\cref{tab:lb-leakage-exclusion}).

\looseness=-1
\textbf{Leakage-free alternatives.}
Eliminating leakage entirely requires a fundamentally different benchmark design.
Two such designs are withheld test data, as in Kaggle-style competitions, and prospective evaluation, where forecasts are submitted before the ground truth exists \citep{meyer2026tsarena,garza2026impermanent}.
Both close the leakage channel, but neither is well suited for model development.
Withheld test data requires rate-limiting submissions and exposes no per-series error breakdown, so researchers cannot see where their model fails, while prospective evaluation requires waiting months for the ground truth to become available.
We therefore view these approaches as complementary to \bench rather than as substitutes for it.






\section{Software package}
\label{sec:code}

Comprehensive task coverage and principled evaluation are essential, but standardized infrastructure is equally important for benchmark relevance and reproducibility. This includes code for task definition, evaluation, and result aggregation.

\textbf{Motivation.}
Existing forecasting benchmarks typically fall into two categories: standalone datasets without supporting infrastructure~\citep{godahewa2021monash,zeng2023transformers}, and end-to-end systems that bundle models, datasets, and forecasting tasks~\citep{qiu2024tfb,aksu2024gift}.
Dataset-only benchmarks provide no guarantee that results are comparable across users, since evaluation may differ in horizons, cutoffs, metrics, or aggregation strategies even on the same dataset.
End-to-end systems reduce this ambiguity but are often rigid and difficult to extend or integrate into existing workflows. They impose many dependencies and assumptions, while model implementations tend to become outdated over time. Version constraints on libraries such as \texttt{torch}, \texttt{numpy}, and \texttt{pandas} can further reduce transparency; for example, changes in \texttt{pandas} may alter frequency inference, seasonal periods $m$, and thus MASE computation.

\textbf{\fev Python library.}
To address these limitations, we introduce \fev, a lightweight library that provides essential benchmarking functionality without unnecessary constraints.
Its core features include task definition, data loading and splitting, prediction scoring, and result aggregation.
\fev depends only on Hugging Face \texttt{datasets}~\citep{lhoest2021datasets} and \texttt{pydantic}~\citep{colvin2025pydantic} for input validation, and does not fix versions of commonly used packages such as \texttt{torch} or \texttt{numpy}.
This enables seamless integration into existing pipelines.
Using \texttt{datasets} also simplifies data loading and supports extensions such as multimodal forecasting.
The \texttt{fev} code is open sourced under the Apache-2.0 license at \url{https://github.com/autogluon/fev}.

\fev deliberately excludes model implementations, which are prone to becoming outdated.
Instead, it provides adapters that convert data into formats compatible with popular forecasting libraries including GluonTS~\citep{alexandrov2020gluonts}, darts~\citep{herzen23darts}, StatsForecast~\citep{garza2022statsforecast}, AutoGluon~\citep{shchur2023autogluon}, and sktime~\citep{loning2019sktime}.

\textbf{Library API}. The \fev library is built around three main constructs.
An \texttt{EvaluationWindow} is a single train–test split defined by a cutoff.
A \texttt{Task} specifies a forecasting problem, including the dataset, horizon, covariates, targets, and evaluation metric, and may include multiple evaluation windows.
Finally, a \texttt{Benchmark} is a collection of tasks.
Benchmarks in \fev can be defined using YAML files.
Each task produces an evaluation summary that includes both metric values and the full task specification, ensuring unambiguous and comparable results.
\fev also provides utilities for aggregating results across tasks, as described in \cref{sec:aggregation}.

\section{Related work}
\label{sec:related-work}
\subsection{Existing benchmarks}

Early deep learning approaches for time series forecasting~\citep{lai2018modeling,salinas2020deepar,rangapuram2018deep} evaluated on small collections of datasets that varied across studies. The M4~\citep{makridakis2020m4} and M5~\citep{makridakis2022m5} competitions later became widely used evaluation tasks. The LTSF benchmark~\citep{zhou2021informer,zeng2023transformers} introduced new datasets with a focus on long-horizon forecasting, but remained narrow in scope and was later criticized for over-representing similar datasets and unrealistic evaluation setups~\citep{hewamalage2023forecast}.

Subsequent benchmarks, including the Monash repository~\citep{godahewa2021monash}, BasicTS+~\citep{shao2024exploring}, TFB~\citep{qiu2024tfb}, ProbTS~\citep{zhang2024probts}, TSFM-bench~\citep{li2025tsfm}, Chronos Benchmarks~\citep{ansari2024chronos}, TIME~\citep{qiao2026time}, and GIFT-Eval~\citep{aksu2024gift}, broadened evaluation across domains, frequencies, and forecasting settings, while domain-specific benchmarks such as BOOM~\citep{cohen2025toto} focused on particular applications. However, these benchmarks largely ignore covariates despite their practical importance.

\looseness=-1
In contrast, \bench provides broad coverage across domains, frequencies, and forecasting settings, including 46 tasks with covariates. Complementing this, the \fev library provides lightweight evaluation infrastructure with standardized aggregation and confidence intervals for reproducible and statistically robust comparisons.



\subsection{Aggregation strategies}
\label{sec:related-work-aggregation}
Several aggregation methods have been proposed for forecasting benchmarks, each with different trade-offs.

\textbf{Average rank} is widely used in forecasting benchmarks
\citep{aksu2024gift,ansari2024chronos}.
As shown in \cref{app:avg-rank-avg-winrate}, average rank is mathematically equivalent to the average win rate $W_j$, inducing the same model ordering. However, ranks scale with the number of models and do not naturally extend to pairwise comparisons. We therefore prefer win rates, which are bounded between $0$ and $1$ and extend directly to pairwise evaluation.

\textbf{Bradley--Terry (Elo) scores} \citep{bradley1952rank} have been applied across domains, from large language models \citep{chiang2024chatbot} to tabular benchmarks \citep{erickson2025tabarena}. As shown in \cref{app:avg-winrate-elo}, when all models are compared on all tasks, Bradley--Terry scores induce the same ranking as average win rates. Thus, average rank, average win rate, and Bradley--Terry scores are equivalent under our setup (\cref{sec:aggregation}), and we report win rates for simplicity.

\looseness=-1
\textbf{Nemenyi post-hoc tests} with critical difference diagrams \citep{demvsar2006statistical} are also rank-based, relying on average ranks. While they control family-wise error rates, they are often conservative and provide only binary significance decisions without effect sizes \citep{garcia2008extension}. In contrast, confidence intervals on win rates directly convey the magnitude and uncertainty of performance differences.

\textbf{Geometric mean relative error (GMRE)}~\citep{ansari2024chronos,aksu2024gift} yields rankings identical to skill scores since $\operatorname{GMRE}_j = 1 - S_j$. We adopt this approach with two modifications: clipping extreme values to reduce outlier influence and reporting its complement (the skill score) to maintain a consistent “higher-is-better” interpretation.

\begin{table*}[t]
    \centering
    \resizebox{0.85\textwidth}{!}{\input{tables/leaderboard_SQL_subset}}
    \caption{Marginal probabilistic forecasting performance of select models (according to the SQL metric) on the full \bench benchmark. Full results are available in \cref{app:extra-results}.\vspace{-3mm}}
    \label{tab:lb-sql-subset}
\end{table*}

\begin{table*}[t]
    \centering
    \resizebox{0.85\textwidth}{!}{\input{tables/leaderboard_MASE_subset}}
    \caption{Marginal point forecasting performance of select models (according to the MASE metric) on the full \bench benchmark. Full results are available in \cref{app:extra-results}.\vspace{-3mm}}
    \label{tab:lb-mase-subset}
\end{table*}

\section{Results}
\label{sec:results}

\begin{figure}
\centering
\subfloat[]{\includegraphics[width=0.52\textwidth]{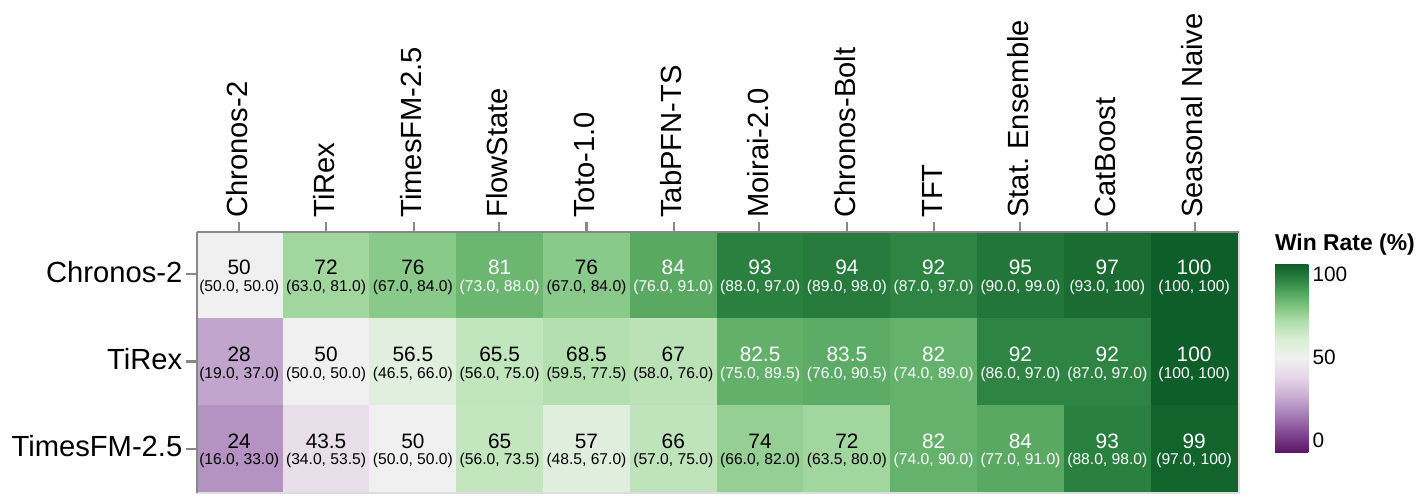}\label{fig:top3-win-rate-pairwise}}
\subfloat[]{\includegraphics[width=0.52\textwidth]{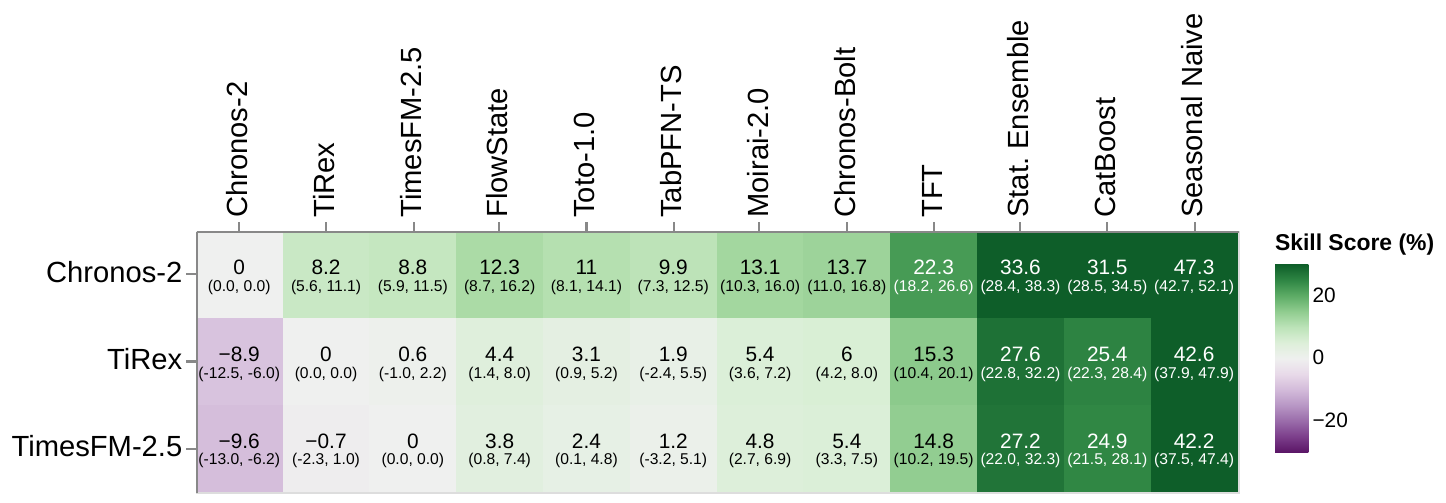}\label{fig:top3-skill-score-pairwise}}

\caption{Pairwise win rates (a) and skill scores (b) of the top-3 models against other models under the SQL metric on \bench, with 95\% confidence intervals obtained via bootstrapping. Higher values are better. Full pairwise results are available in \cref{app:extra-results}. Best viewed on screen.}
\end{figure}


\textbf{Models.} Our evaluation focuses primarily on pretrained forecasting models that represent the current frontier in time series forecasting research.
We select models based on three criteria: strong performance on existing benchmarks such as GIFT-Eval, publicly available implementations, and computational feasibility on consumer hardware (single NVIDIA A10G GPU with 24GB RAM).
More details about the model configuration are provided in \cref{app:models}.

We evaluate nine pretrained models: Chronos-2 \citep{ansari2025chronos2}, TimesFM-2.5 \citep{das2024decoder}, TiRex \citep{auer2025tirex}, FlowState \citep{graf2025flowstate}, Chronos-Bolt \citep{ansari2024chronos}, Toto \citep{cohen2025toto}, Moirai-2.0 \citep{woo2024unified}, TabPFN-TS \citep{hoo2025tables}, and Sundial \citep{liu2025sundial}. Chronos-2 and Toto-1.0 natively support multivariate forecasting ($D > 1$); all other models forecast each dimension independently. Chronos-2 supports past and future covariates, TabPFN-TS supports known covariates, and the remaining pretrained models ignore covariates.

We additionally evaluate statistical models (AutoETS, AutoARIMA, AutoTheta, and the SCUM ensemble) \citep{garza2022statsforecast,petropoulos2020simple}, simple baselines (Seasonal Naive, Naive, Drift) \citep{hyndman2018forecasting}, and task-specific models including CatBoost \citep{prokhorenkova2018catboost}, LightGBM \citep{ke2017lightgbm}, PatchTST \citep{nie2022time}, TFT \citep{lim2021temporal}, and DeepAR \citep{salinas2020deepar}.

\textbf{Evaluation metrics.}
We follow the evaluation protocol from Sections~\ref{sec:tasks} and \ref{sec:aggregation}. Tasks are evaluated using SQL for probabilistic forecasting and MASE for point forecasting, with results aggregated across all 100 tasks using win rates and skill scores for both marginal and pairwise comparisons. For broader context, we also report the following quantities.

\textbf{Data leakage.}
For pretrained models, results on tasks with self-reported training overlap are imputed following the policy in \cref{sec:leakage}.

\textbf{Runtime.}
We report the median end-to-end runtime (training and inference across all evaluation windows) normalized to 100 time series per task by scaling runtimes by $100 / (N \times W)$, where $N$ is the number of series and $W$ the number of evaluation windows. Although runtime depends on hardware and implementation details, it provides a useful measure of computational efficiency across modeling paradigms and training setups.

\textbf{Model failures.} If a model fails to produce a forecast on certain tasks (e.g., due to crashing), we replace its performance with the score of the Seasonal Naive baseline.

\subsection{Results}
Tables~\ref{tab:lb-sql-subset} \& \ref{tab:lb-mase-subset} summarize the marginal performance of selected models on \bench, showing the strongest models from each category for brevity. Full results for all 21 models are provided in the appendix, and the live leaderboard is available on Hugging Face.
Chronos-2 achieves the highest average win rate and skill score values, followed by TiRex and TimesFM-2.5. Older pretrained models such as Chronos-Bolt and Toto rank lower. All pretrained models are substantially more accurate and faster than the best statistical and task-specific models (TFT and SCUM Ensemble).

\looseness=-1
To assess whether differences among the leading models reflect genuine improvements rather than evaluation noise, we examine pairwise comparisons under the SQL metric in Figures~\ref{fig:top3-win-rate-pairwise} and \ref{fig:top3-skill-score-pairwise}. The confidence intervals indicate a statistically significant gap between Chronos-2 and the remaining models. In contrast, TiRex and TimesFM-2.5 show heavy overlap, suggesting no clear winner; under different benchmark compositions or task weightings, either could emerge as the second-best performer.
At the same time, the overall ranking is not driven by the domains or frequencies that contribute the most tasks to \bench: when every domain, or alternatively every frequency group, receives equal total weight, the ranking of the models remains nearly identical (\cref{app:robustness}).

\begin{figure}[t]
\centering
\begin{minipage}{0.48\columnwidth}
    \centering
    \includegraphics[width=\linewidth]{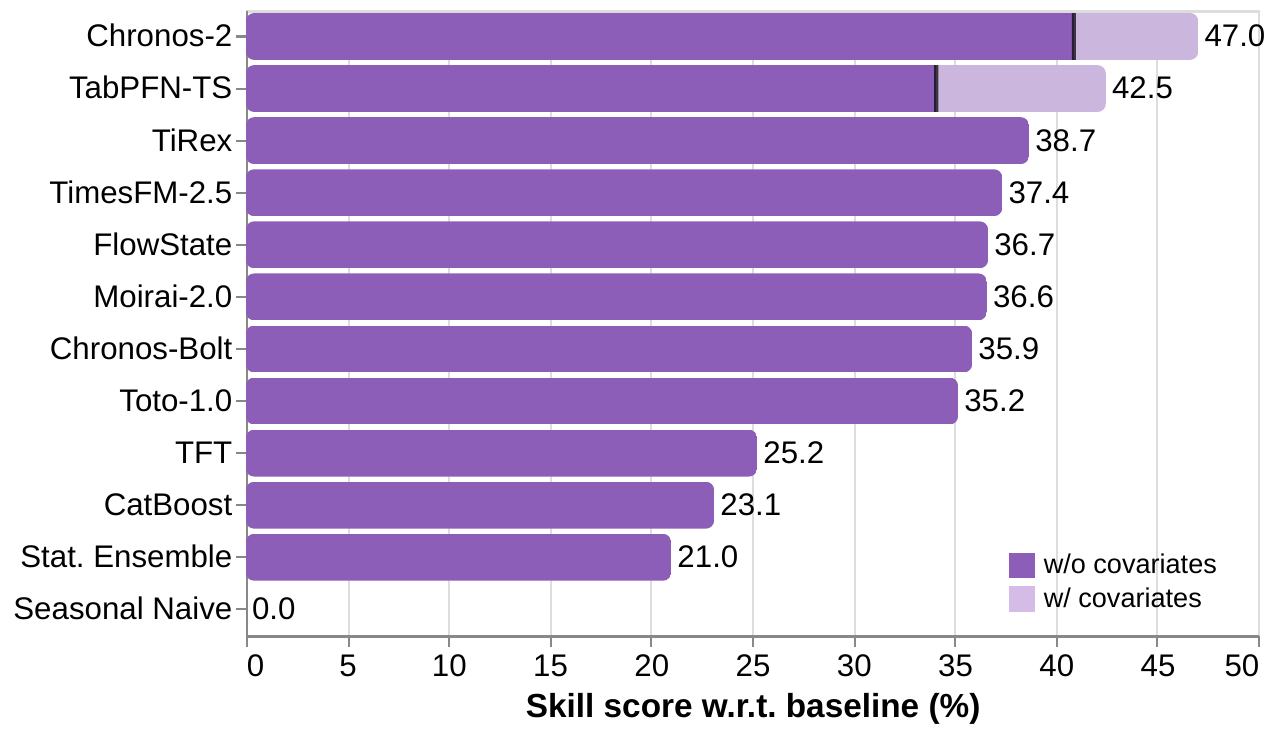}
    \caption{Average skill scores (SQL) on the 42 \bench\ tasks with dynamic covariates.
    }
    \label{fig:skill-score-covariates}
\end{minipage}
\hfill
\begin{minipage}{0.48\columnwidth}
    \centering
    \includegraphics[width=\linewidth]{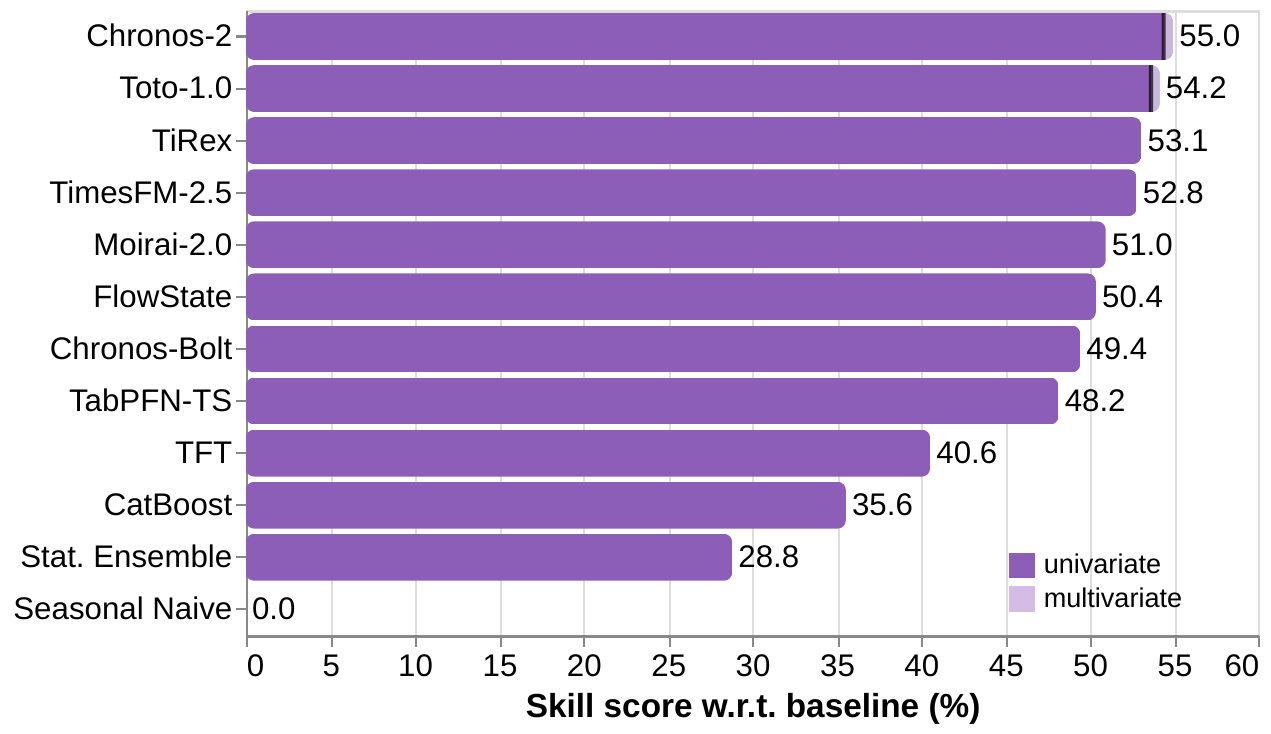}
    \caption{Average skill scores (SQL) on 35 multivariate tasks of \bench.
    }
    \label{fig:skill-score-only-multivariate}
\end{minipage}
\end{figure}

\subsection{Covariates and multivariate forecasting}

\looseness=-1
\textbf{Forecasting with covariates.}
To measure the impact of covariates, we compare models with and without covariate inputs on the 42 \bench tasks containing dynamic covariates (\cref{fig:skill-score-covariates}), reporting skill scores under the SQL metric. Covariates substantially improve performance for both Chronos-2 (47.0\% vs.\ 40.9\%) and TabPFN-TS (42.5\% vs.\ 34.1\%), showing that covariate information is a key driver of their strong results. Yet most pretrained forecasting models still ignore covariates, leaving substantial gains untapped and making covariate support an important direction for future work.

\looseness=-1
\textbf{Multivariate forecasting.}
To measure the impact of native multivariate modeling, we compare models in multivariate and univariate modes on the 35 multivariate tasks in \bench\ (\cref{fig:skill-score-only-multivariate}), again using SQL skill scores. The gains from multivariate modeling are modest for both Chronos-2 (55.0\% vs.\ 54.4\%) and Toto-1.0 (54.7\% vs.\ 54.1\%). Similar findings were reported by \citet{nie2022time}, who observed that univariate models often match multivariate ones. In contrast to the large gains from covariates, most forecasting benchmarks emphasize multivariate forecasting while largely ignoring covariates, highlighting the gap addressed by \bench.

\section{Scope and limitations}
\label{sec:limitations}
\looseness=-1
\bench is designed as a broad benchmark spanning many domains, frequencies, and forecasting settings, providing a general view of model capabilities rather than optimizing for any single application domain. As a result, some models may perform differently on specialized tasks or domain-specific benchmarks, and we hope the lightweight design of \fev will facilitate the creation of complementary benchmarks for specific use cases.
Since we selected tasks by application domain following surveys of forecasting practice \citep{petropoulos2022forecasting}, some domains and frequencies are represented by more tasks than others, although the model rankings remain nearly identical when every domain or frequency group receives equal weight (\cref{app:robustness}).
We deliberately exclude sub-minute frequencies, since the tasks studied on such data (e.g., EEG or industrial sensors) are typically anomaly detection or classification rather than forecasting.
Finally, \bench prioritizes broad coverage, fast evaluation, and suitability for model development over fully leakage-free evaluation, and therefore relies on self-reported training overlap, as discussed in \cref{sec:leakage}.

\section{Conclusion}
We introduced \bench, a benchmark with covariates, multivariate tasks, and principled aggregation methods for statistically robust evaluation across diverse forecasting domains. Our results highlight covariate support as a promising direction for pretrained forecasting models. Complementing this, we presented \fev, a lightweight package for reproducible and extensible forecasting evaluation. Together, \bench and \fev provide a practical foundation for developing and comparing the next generation of forecasting models.

\bibliography{references}
\bibliographystyle{icml2025}


\appendix

\appendix
\onecolumn  
\section{Tasks}
\label{app:tasks}
In total, \bench 
contains 100 time series forecasting tasks.
In this section we provide the main statistics of these tasks together with citations to the sources of the datasets.
For competition datasets we use their fixed forecast horizon $H$; for all others, $H$ is set by a frequency--horizon mapping, except for a subset of hourly datasets where we use $H=168$ to enable long-range forecasting.
The number of evaluation windows $W$ is then chosen to evenly split the series while ensuring that sufficient historical data is available for each forecast of length $H$.
Dataset frequencies are reported using \texttt{pandas} frequency aliases (minu\textbf{T}ely, \textbf{H}ourly, \textbf{D}aily, \textbf{W}eekly, \textbf{M}onthly, \textbf{Q}uarterly, \textbf{Y}early).

The precise task definitions in YAML format are available under \href{https://github.com/autogluon/fev}{\texttt{https://github.com/autogluon/fev}}.
The datasets used for evaluation under \href{https://huggingface.co/datasets/autogluon/fev_datasets}{\texttt{https://huggingface.co/datasets/autogluon/fev\_datasets}}.

\subsection{GIFT-Eval}

\begin{table}[h]
\resizebox{\textwidth}{!}{\input{tables/task_stats/gift}}
    \caption{Tasks based on datasets coming from the GIFT-Eval corpus \citep{aksu2024gift}.}
\end{table}

The GIFT-Eval corpus \citep{aksu2024gift} contains various univariate and multivariate datasets, none of which provide covariates.
The original datasets have been collected from sources such as \citet{godahewa2021monash,jiang2023libcity,mancuso2021machine,wu2021autoformer,palaskar2024automixer}.

\subsection{Macroeconomic datasets}

\begin{table}[h]
\resizebox{\textwidth}{!}{\input{tables/task_stats/macro}}
    \caption{Tasks based on various macroeconomic datasets.}
\end{table}

We consider various macroeconomic datasets such as GVAR \citep{mohaddes2024gvar}, US Consumption \citep{wilms2016consumption}, Australian Tourism \citep{athanasopoulos2009hierarchical}, FRED-MD \citep{mccracken2016fred}, FRED-QD \citep{mccracken2020fred}, world CO2 emmissions \citep{Kaggle2025CO2EmissionsByCountry}, life expectancy \citep{Kaggle2025GlobalLifeExpectancy1950_2023} and global tourism \citep{Kaggle2025TourismEconomicImpact}.

For each of FRED-MD and FRED-QD, we create two forecasting tasks.
The first follows the CEE model \citep{christiano1999monetary} and focuses on forecasting employment, inflation, and federal funds rate indicators.
The second task involves jointly forecasting 51 core macroeconomic indicators.
Note that we use the snapshot of FRED-MD corresponding to August 2025, which is different from FRED-MD snapshot used in \citet{godahewa2021monash}.

\subsection{Energy datasets}

\begin{table}[h]
\resizebox{\textwidth}{!}{\input{tables/task_stats/energy}}
    \caption{Tasks based on datasets related to energy generation and consumption.}
\end{table}

These datasets include electricity price forecasting (EPF) benchmark \citep{lago2021forecasting}, ERCOT generation data \citep{ansari2024chronos}, ENTSO-e load data \citep{OPSD2020} with weather originating from \texttt{Renewables.ninja} \citep{staffell2023global}, and solar generation with weather \citep{Kaggle2025RenewableEnergyWeather}.

\subsection{BOOMLET}

\begin{table}[h]
\resizebox{\textwidth}{!}{\input{tables/task_stats/boomlet}}
    \caption{Tasks based on datasets from BOOMLET \citep{cohen2025toto}.}
\end{table}

We include the multivariate observability datasets from the BOOMLET benchmark \citep{cohen2025toto}. BOOMLET is a subset of the larger BOOM benchmark curated by the original authors. We additionally limit our attention to datasets with frequency of at least 1 minute to avoid including too many datasets from a single source to \bench.

\subsection{Forecasting competitions}

\begin{table}[h]
\resizebox{\textwidth}{!}{\input{tables/task_stats/competitions}}
    \caption{Tasks based on datasets coming from various forecasting competitions}
\end{table}

We use datasets from forecasting competitions held on \texttt{kaggle.com} \citep{bojer2021kaggle}.
These include Favorita store sales \& transactions \citep{Kaggle2020StoreSales}, the M5 competition \citep{makridakis2022m5}, restaurant visitor \& reservation \citep{Kaggle2017RecruitRestaurant}, Rossmann \citep{Kaggle2015Rossmann}, Walmart \citep{Kaggle2014Walmart}, and Rohlik \citep{RohlikSalesForecasting2025} store sales forecasting competitions.
We also consider the KDD Cup 2022 dataset where the goal is to predict wind power generation \citep{zhou2022sdwpf}, and the Global Energy Forecasting Competitions held in 2012, 2014 and 2017 \citep{hong2014global}.

\subsection{Other sources}

\begin{table}[h]
\resizebox{\textwidth}{!}{\input{tables/task_stats/other}}
    \caption{Tasks based on datasets collected from other sources.}
\end{table}

We also include datasets from the following miscellaneous sources.
\begin{itemize}
\item Influenza-like-illness cases collected by the European Centre for Disease Prevention and Control \citep{ECDC2025RespiratoryViruses}.

\item Fashion trend data from Hermes \citep{david2022hermes}.

\item Hospital admissions data from Riyadh \citep{Kaggle2025RiyadhHospitalAdmissions}.

\item Query counts for Amazon Redshift database servers \citep{renen2024redset}.

\item Solar energy generation with corresponding weather covariates \citep{Kaggle2025RenewableEnergyWeather}.

\item Air quality measurements in an Italian city with accompanying weather data \citep{DeVito2008ElectronicNoseCalibration}.

\item COVID-19 cases, hospital admissions, and deaths in the United Kingdom at different administrative levels \citep{Kaggle2025UKCovidDashboard}.
\end{itemize}



\section{Models}
\label{app:models}

We evaluate nine pretrained models on \bench, whose key properties are summarized in \cref{tab:pretrained-models-summary}. Most are decoder-only transformers, except TiRex (decoder-only xLSTM), Chronos-Bolt Base (encoder-decoder transformer), Chronos-2 (encoder-only transformer), and FlowState (SSM encoder with functional basis decoder).
All models except TabPFN-TS and FlowState process non-overlapping patches of time series rather than individual observations.
Toto, TabPFN-TS, and Sundial produce sample forecasts, while the others produce direct multi-step-ahead quantile forecasts.
Toto and Chronos-2 are the only models that natively support multivariate targets.
We provided only known covariates to TabPFN-TS, excluding past covariates as recommended by the authors. Chronos-2 supports both past and known covariates.

For all pretrained models, we keep hyperparameters at their default values unless specified in the table. For Toto, we reduced the samples per batch and batch size to avoid out-of-memory errors on large multivariate datasets. We run all pretrained models on a g5.2xlarge AWS instance with a single A10G GPU (24GB GPU RAM, 32GB RAM), using PyTorch 2.6 with CUDA 12.6 via AWS Deep Learning Containers. The version in the table below refers to either the PyPI package version, or the date on which the official repository was cloned if no PyPI package is provided by the authors.

\begin{table}[h]
    \centering
    \resizebox{\textwidth}{!}{%
    \input{tables/model_properties}
    }
    \caption{Properties of different pretrained time series forecasting models.}
    \label{tab:pretrained-models-summary}
\end{table}

We include statistical baselines from the StatsForecast \texttt{v2.0.1} \citep{garza2022statsforecast}, such as AutoETS, AutoARIMA, AutoTheta, as well as the SCUM
ensemble~\citep{petropoulos2020simple}. To avoid long runtimes, we truncate the context length of statistical models to 1000 steps and set the maximum season length to 200 for AutoETS,
AutoTheta, AutoARIMA and SCUM Ensemble.

We train LightGBM~\citep{ke2017lightgbm} and CatBoost~\citep{prokhorenkova2018catboost} in recursive mode using MLForecast \texttt{v0.14.0} \citep{morales2021mlforecast}, with
preprocessing hyperparameters (target transforms, lag transforms) selected via grid search. We also train DeepAR \citep{salinas2020deepar}, PatchTST \citep{nie2022time}, and TFT \citep{lim2021temporal} using AutoGluon-TimeSeries \texttt{v1.5.0} \citep{shchur2023autogluon}, refitting on each evaluation window. We use seasonal naive as a fallback when there is insufficient training data for deep learning models in some windows.

Statistical and global models were run on \texttt{m6i.4xlarge} AWS instances with 16 vCPU cores and 64GB RAM. All model wrappers are available at \href{https://github.com/autogluon/fev/tree/main/models}{\texttt{github.com/autogluon/fev/models}}.

\section{Extended discussion of aggregation methods}
\label{app:agg-methods}
This appendix clarifies how average win rates $W_j$ (\cref{eq:avg-winrate}) relate to other aggregation methods
commonly used in benchmarking, such as average ranks \citep{aksu2024gift,ansari2024chronos} and Bradley--Terry (``Elo'') scores \citep{erickson2025tabarena}.
We show that all three induce the same ordering of models.

\subsection{Average win rate and average rank are equivalent}
\label{app:avg-rank-avg-winrate}
For a given task $r$ and model $j$, let
\[
M_{\text{lower}}=\sum_{\substack{k=1 \\ k \neq j}}^{M}\mathds{1}(E_{rk}<E_{rj}),
\qquad
M_{\text{tied}}=\sum_{\substack{k=1 \\ k \neq j}}^{M}\mathds{1}(E_{rk}=E_{rj}).
\]
The midrank of $j$ on task $r$ is
\[
\operatorname{rank}_{rj} \;=\; 1 + M_{\text{lower}} + \tfrac{1}{2}M_{\text{tied}}.
\]

Its contribution to $W_j$ equals
\[
\frac{1}{M-1}\sum_{k \ne j} \Big(\mathds{1}(E_{rj}<E_{rk})+\tfrac{1}{2}\,\mathds{1}(E_{rj}=E_{rk})\Big)
=1-\frac{\operatorname{rank}_{rj}-1}{M-1}.
\]
Averaging over tasks gives
\[
W_j = 1-\frac{\overline{\operatorname{rank}}_j-1}{M-1},
\qquad
\overline{\operatorname{rank}}_j=\tfrac{1}{R}\sum_{r=1}^R \operatorname{rank}_{rj}.
\]

Thus, $W_j$ and $\overline{\operatorname{rank}}_j$ are affinely equivalent and induce the same ordering of models (with lower rank $\leftrightarrow$ higher win rate).

\subsection{Average win rate and Bradley--Terry (Elo) scores result in the same ranking}
\label{app:avg-winrate-elo}
The Bradley--Terry (BT) model, also known as Elo rating \citep{chiang2024chatbot}, provides
a parametric way to convert pairwise win rates into latent skill scores. In contrast to the
nonparametric average win rate $W_j$, the BT model assumes that each model $j$ has an
underlying skill parameter $\theta_j \in \R$, and that the probability of $j$ outperforming $k$
follows a logistic link:
\[
\Pr(E_{rj} < E_{rk}) \;=\; \sigma\!\big(\lambda(\theta_j-\theta_k)\big),
\qquad
\sigma(x)=\tfrac{1}{1+e^{-x}}, \;\; \lambda>0.
\]
Here $\lambda$ is a scaling constant (in Elo, $\lambda=\ln 10/400$). The parameters
$\theta=(\theta_1,\dots,\theta_M)$ are estimated by maximum likelihood:
\[
\hat{\theta} \;\in\; \arg\max_{\theta \in \R^M}
\sum_{j<m}\Big[
W_{jm}\log\sigma(\lambda(\theta_j-\theta_m))
+(1-W_{jm})\log\sigma(\lambda(\theta_m-\theta_j))
\Big].
\]
Typically some identifiability constraint is added, such as fixing $\theta_\beta = 1000$ for a chosen baseline $\beta$.

\begin{proposition}
Suppose all $M$ models are compared on the same $R$ tasks, with pairwise win rates $W_{jk}$
(\cref{eq:pairwise-winrate}) and average win rates
$W_j=\tfrac{1}{M-1}\sum_{k\neq j} W_{jk}$ (\cref{eq:avg-winrate}).
At the BT MLE with scale $\lambda>0$,
\[
\theta_j>\theta_k \;\;\Longleftrightarrow\;\; W_j>W_k,
\qquad
\theta_j=\theta_k \;\;\Longleftrightarrow\;\; W_j=W_k.
\]
\end{proposition}

\begin{proof}
Differentiating the log-likelihood gives the score equations
\[
\frac{\partial \ell}{\partial \theta_j}
=\lambda\sum_{m\neq j}\Big(W_{jm}-\sigma(\lambda(\theta_j-\theta_m))\Big)=0.
\]
Subtracting the equations for $j$ and $k$ yields
\[
(M-1)(W_j-W_k)=
\sum_{m\neq j,k}\big[\sigma(\lambda(\theta_j-\theta_m))-\sigma(\lambda(\theta_k-\theta_m))\big]
+\big[\sigma(\lambda\Delta)-\sigma(-\lambda\Delta)\big],
\]
where $\Delta=\theta_j-\theta_k$. Each term on the right is strictly increasing in $\Delta$, so
the whole expression has the same sign as $\Delta$. Thus
$\operatorname{sign}(W_j-W_k)=\operatorname{sign}(\theta_j-\theta_k)$, with equality iff
$\Delta=0$. Strict concavity of the BT log-likelihood ensures uniqueness of the solution up to translation.
\end{proof}

\textbf{Conclusion.} Average win rate $W_j$ and BT/Elo scores $\theta_j$ induce the same ordering of models,
with higher win rate corresponding to higher Elo score.



\newpage

\section{Extended results}
\label{app:extra-results}
Tabular results for each dataset-task combination, updated marginal and pairwise results in an interactive format are available at \href{https://huggingface.co/spaces/autogluon/fev-bench}{\texttt{huggingface.co/spaces/autogluon/fev-bench}}.
The raw results in CSV format are available at \href{https://github.com/autogluon/fev/tree/main/benchmarks/fev_bench/results}{\texttt{github.com/autogluon/fev}}.

The failures for Stat.\ Ensemble, AutoARIMA, and AutoETS models in the tables below correspond to the models exceeding the 6 hour time limit for a single task.
TabPFN-TS failed on 1 task due to out of memory errors, and DeepAR failed due to encountering NaNs during training.

\subsection{Marginal performance}

\begin{table}[h]
    \centering
    \resizebox{\textwidth}{!}{\input{tables/leaderboard_SQL}}
    \caption{Marginal probabilistic forecasting performance of all models (according to the SQL metric) on the full \bench benchmark. The reported metrics are defined in Sections~\ref{sec:aggregation-marginal} and \ref{sec:results}.}
    \label{tab:lb-sql}
\end{table}

\begin{table}[h]
    \centering
    \resizebox{\textwidth}{!}{\input{tables/leaderboard_MASE}}
    \caption{Marginal point forecasting performance of all models (according to the MASE metric) on the full \bench benchmark. The reported metrics are defined in Sections~\ref{sec:aggregation-marginal} and \ref{sec:results}.}
    \label{tab:lb-mase}
\end{table}

\subsection{Robustness of the rankings}
\label{app:robustness}
\looseness=-1
In this section we check whether the rankings reported in \cref{tab:lb-sql} depend on two design choices of \bench: the uniform weighting of tasks, which may over-represent some domains and frequencies, and the imputation of tasks with reported leakage.
For brevity, we report results only for the SQL metric and for the subset of models shown in the main text.

\textbf{Frequency- and domain-balanced weighting.}
\bench contains more tasks from some domains (e.g., energy, retail and cloud) than from others (e.g., mobility and nature), and more tasks with minutely, hourly and daily frequency than with monthly, quarterly or yearly frequency.
One may therefore ask whether the rankings are driven by the dominant domains or frequencies.
To answer this question, we generalize the average win rate (\cref{eq:avg-winrate}) and skill score (\cref{eq:avg-skillscore}) to weighted averages over tasks, where task $r$ receives weight $w_r$ with $\sum_{r=1}^R w_r = 1$, so that uniform weighting corresponds to the special case $w_r = \nicefrac{1}{R}$.
Let the tasks be partitioned into $G$ groups, and let $g(r)$ denote the group of task $r$ and $n_{g}$ the number of tasks in group $g$.
We then set $w_r = \nicefrac{1}{(G \cdot n_{g(r)})}$, such that each group receives the same total weight $\nicefrac{1}{G}$.
We consider two such groupings: by the $G = 7$ domains from \cref{tab:dataset-domains}, and by the $G = 6$ frequency groups T, H, D, W, M, and Q+Y, where all minutely frequencies (e.g., 5T, 10T, 15T) form a single group T, and quarterly and yearly tasks are merged into a single group because of their small number.
As shown in \cref{tab:lb-balanced}, the ranking of the models remains nearly identical under all three weighting schemes.
We nevertheless retain uniform weighting for the main results, as this is the standard practice in existing benchmarks.

\textbf{Excluding leaked tasks.}
As discussed in \cref{sec:leakage}, \bench imputes the scores of pretrained models on tasks with reported training overlap with the scores of Chronos-Bolt (Base).
To verify that our conclusions do not depend on this imputation, we instead drop the 10 tasks that are flagged as leaky for at least one model, and recompute the scores for all models on the remaining 90 tasks.
We omit Moirai-2.0 from this comparison, since its reported overlap covers too many tasks to be dropped without losing a large fraction of the benchmark.
As shown in \cref{tab:lb-leakage-exclusion}, the ranking of the models on the leakage-free tasks is unchanged compared to the full set of tasks.

\begin{table}[h]
    \centering
    \resizebox{\textwidth}{!}{%
\begin{tabular}{lrrrrrr}
\toprule
 & \multicolumn{2}{c}{\textbf{Uniform}} & \multicolumn{2}{c}{\textbf{Frequency-balanced}} & \multicolumn{2}{c}{\textbf{Domain-balanced}} \\
\cmidrule(lr){2-3} \cmidrule(lr){4-5} \cmidrule(lr){6-7}
\textbf{Model} & \textbf{Win rate (\%)} & \textbf{Skill score (\%)} & \textbf{Win rate (\%)} & \textbf{Skill score (\%)} & \textbf{Win rate (\%)} & \textbf{Skill score (\%)} \\
\midrule
Chronos-2 & 91.4 & 47.3 & 90.8 & 44.1 & 91.2 & 45.3 \\
TiRex & 83.3 & 42.6 & 84.1 & 40.1 & 84.1 & 41.7 \\
TimesFM-2.5 & 79.6 & 42.2 & 77.4 & 39.0 & 80.2 & 41.3 \\
FlowState & 73.3 & 39.9 & 74.2 & 37.6 & 74.5 & 39.2 \\
Toto-1.0 & 72.5 & 40.8 & 69.1 & 37.2 & 71.8 & 39.9 \\
TabPFN-TS & 69.7 & 41.5 & 70.6 & 38.9 & 68.8 & 39.7 \\
Moirai-2.0 & 67.5 & 39.3 & 64.1 & 35.7 & 67.3 & 38.4 \\
Chronos-Bolt & 67.0 & 38.9 & 65.8 & 36.1 & 67.5 & 38.1 \\
TFT & 51.4 & 32.2 & 45.5 & 26.9 & 48.7 & 28.8 \\
Stat. Ensemble & 45.5 & 20.7 & 52.7 & 23.2 & 47.6 & 21.5 \\
CatBoost & 34.4 & 23.0 & 31.8 & 20.5 & 33.0 & 20.6 \\
Seasonal Naive & 19.7 & 0.0 & 19.4 & 0.0 & 19.8 & 0.0 \\
\bottomrule
\end{tabular}%
    }
    \caption{Marginal probabilistic forecasting performance (according to the SQL metric) on \bench with uniform task weights, frequency-balanced task weights (each of the 6 frequency groups receives the same total weight), and domain-balanced task weights (each of the 7 domains receives the same total weight).}
    \label{tab:lb-balanced}
\end{table}

\begin{table}[h]
    \centering
    \resizebox{0.75\textwidth}{!}{%
\begin{tabular}{lrrrr}
\toprule
 & \multicolumn{2}{c}{\textbf{All tasks}} & \multicolumn{2}{c}{\textbf{Leakage-free tasks}} \\
\cmidrule(lr){2-3} \cmidrule(lr){4-5}
\textbf{Model} & \textbf{Win rate (\%)} & \textbf{Skill score (\%)} & \textbf{Win rate (\%)} & \textbf{Skill score (\%)} \\
\midrule
Chronos-2 & 91.3 & 47.3 & 90.8 & 47.0 \\
TiRex & 83.3 & 42.6 & 83.2 & 43.1 \\
TimesFM-2.5 & 79.9 & 42.2 & 81.2 & 43.0 \\
FlowState & 74.2 & 39.9 & 74.5 & 40.4 \\
Toto-1.0 & 73.2 & 40.8 & 73.7 & 41.4 \\
TabPFN-TS & 70.6 & 41.5 & 71.6 & 42.2 \\
Chronos-Bolt & 68.1 & 38.9 & 68.1 & 39.3 \\
TFT & 52.4 & 32.2 & 52.3 & 32.6 \\
Stat. Ensemble & 46.7 & 20.7 & 46.0 & 20.1 \\
CatBoost & 35.7 & 23.0 & 34.7 & 22.4 \\
Seasonal Naive & 20.4 & 0.0 & 20.7 & 0.0 \\
\bottomrule
\end{tabular}%
    }
    \caption{Marginal probabilistic forecasting performance (according to the SQL metric) on all 100 tasks of \bench and on the 90 tasks without reported leakage for any model. Moirai-2.0 is omitted from both columns, so the scores on all tasks differ slightly from \cref{tab:lb-sql}.}
    \label{tab:lb-leakage-exclusion}
\end{table}

\clearpage
\newpage
\subsection{Pairwise comparison}

\begin{figure}[h]
\centering
\includegraphics[width=\textwidth]{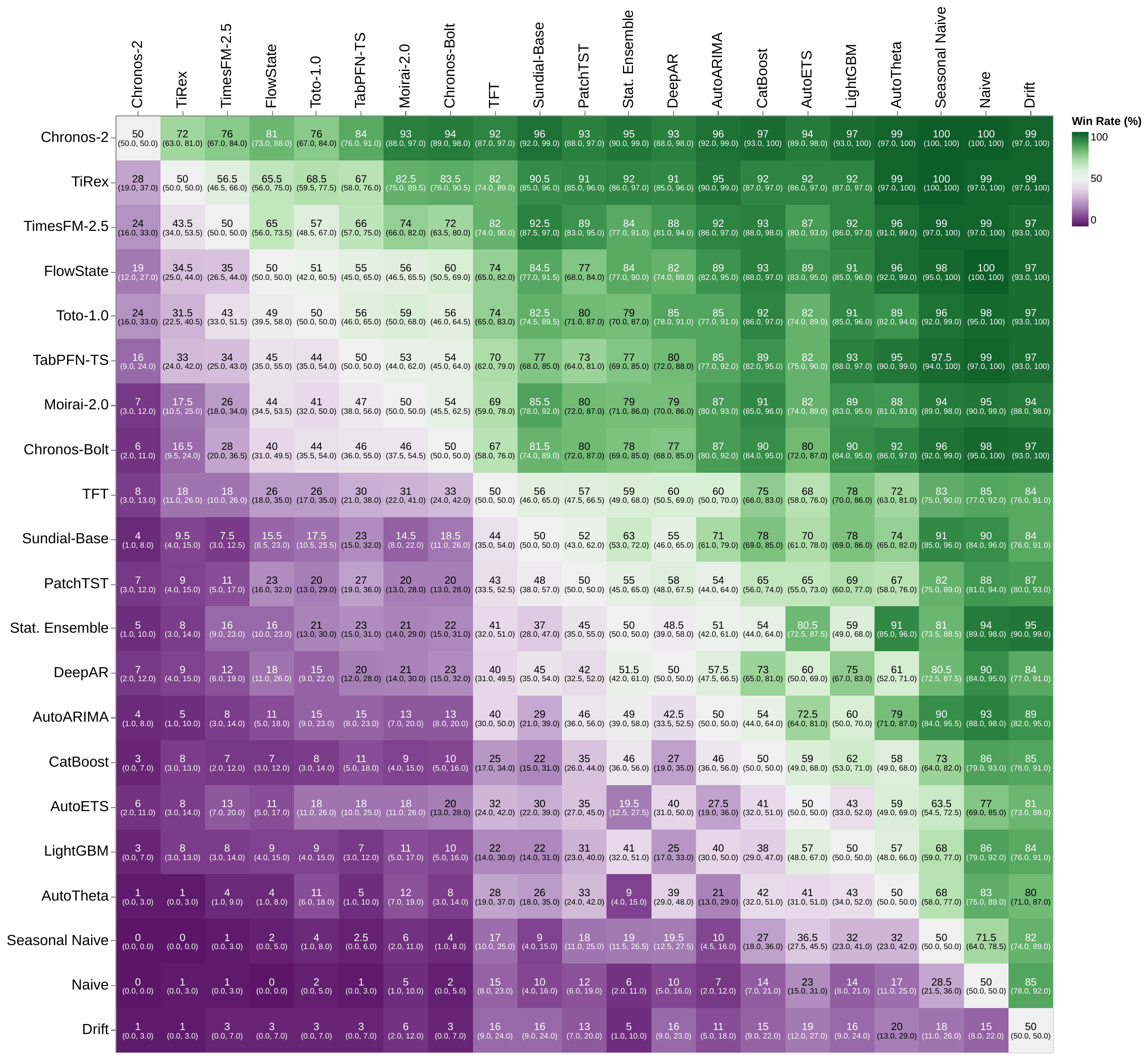}

\caption{Pairwise win rates $W_{jk}$ (\cref{eq:pairwise-winrate}) of all models against each other under the scaled quantile loss (SQL) metric on \bench, with 95\% confidence intervals obtained via bootstrapping. Higher values are better. Best viewed on screen.}
\end{figure}

\newpage

\begin{figure}
\centering
\includegraphics[width=\textwidth]{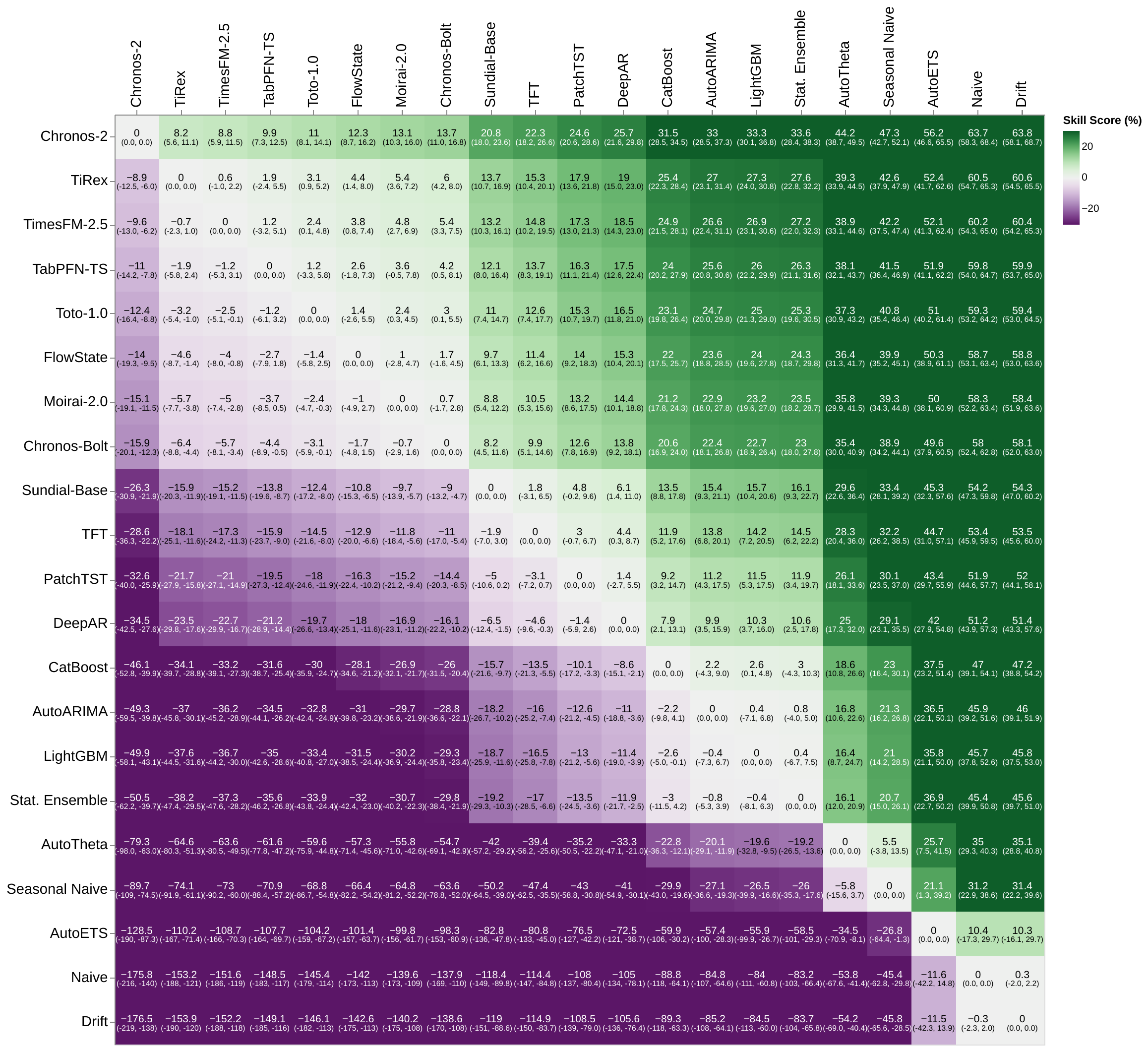}

\caption{Pairwise skill scores $S_{jk}$ (\cref{eq:pairwise-skillscore}) of all models against each other under the scaled quantile loss (SQL) metric on \bench, with 95\% confidence intervals obtained via bootstrapping. Higher values are better.
Note that pairwise skill score is not symmetric, $S_{jk} \ne S_{kj}$.  Best viewed on screen.
}\end{figure}

\begin{figure}[h]
\centering
\includegraphics[width=\textwidth]{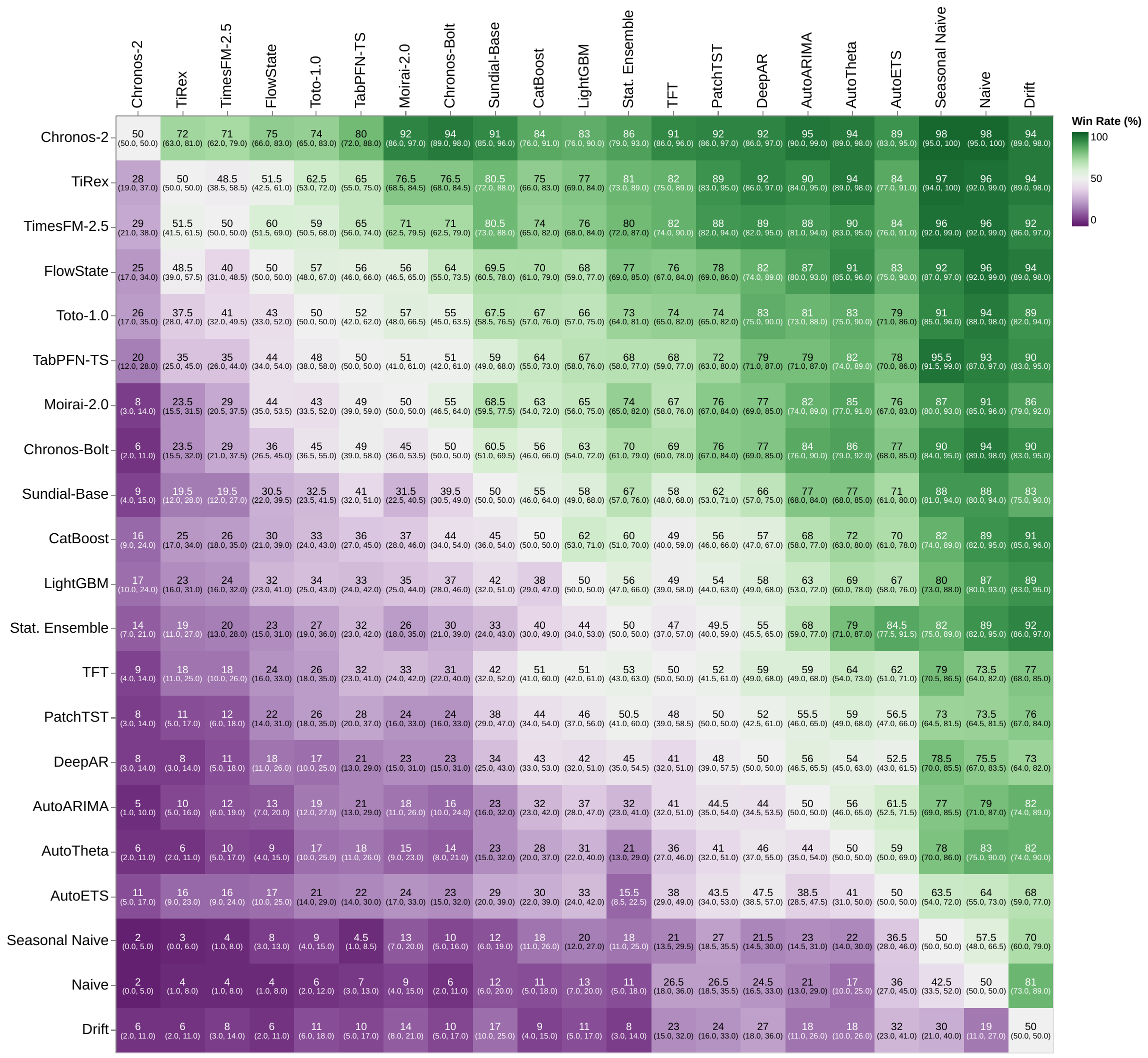}

\caption{Pairwise win rates $W_{jk}$ (\cref{eq:pairwise-winrate}) of all models against each other under the mean absolute scaled error (MASE) metric on \bench, with 95\% confidence intervals obtained via bootstrapping. Higher values are better.  Best viewed on screen.}
\end{figure}

\begin{figure}[h]
\centering
\includegraphics[width=\textwidth]{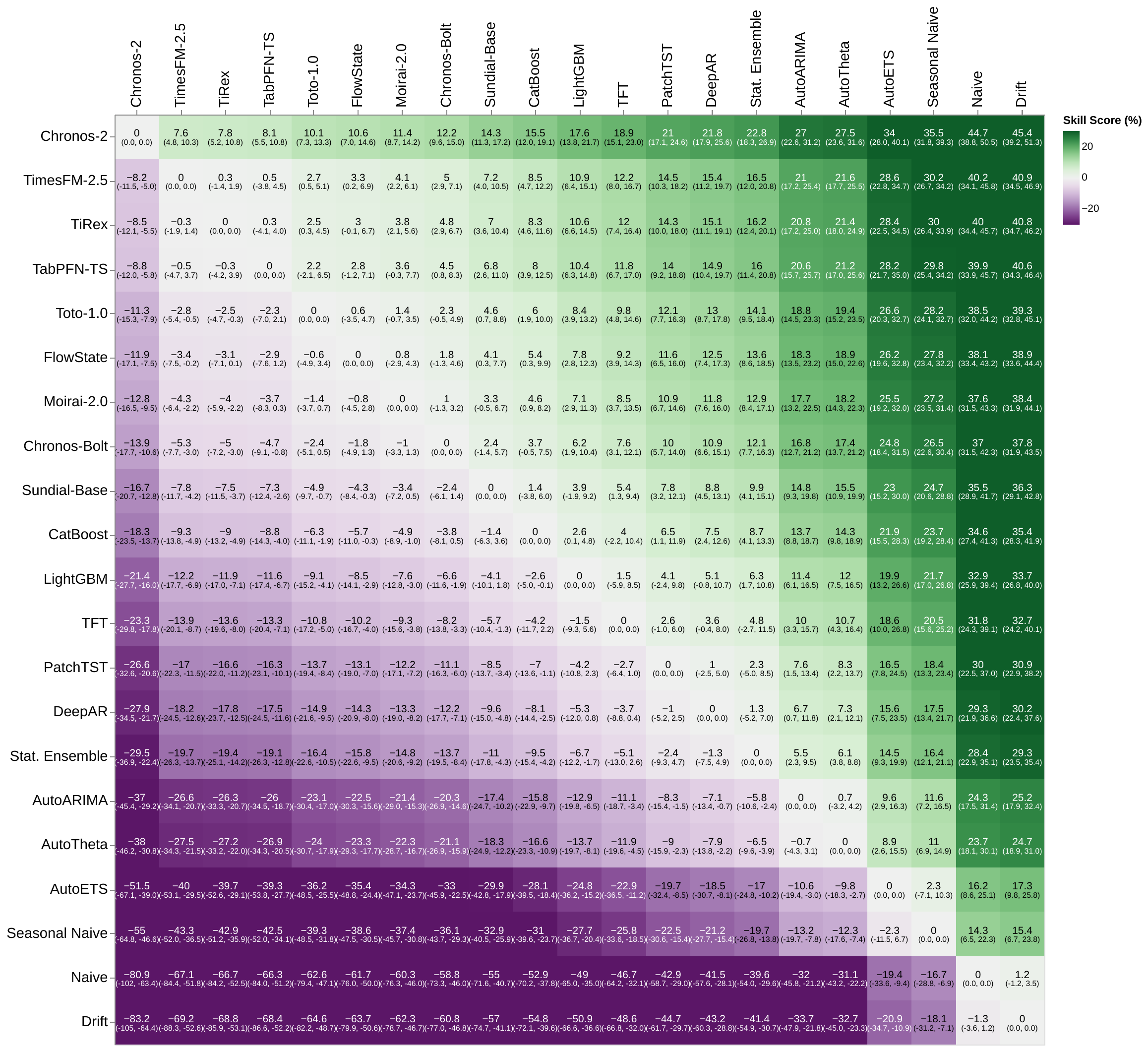}

\caption{Pairwise skill scores $S_{jk}$ (\cref{eq:pairwise-skillscore}) of all models against each other under the mean absolute scaled error (MASE) metric on \bench, with 95\% confidence intervals obtained via bootstrapping. Higher values are better.
Note that pairwise skill score is not symmetric, $S_{jk} \ne S_{kj}$.  Best viewed on screen.
}\end{figure}

\clearpage

\section{\minibench: A representative subset of \bench}
\label{app:mini-bench}

\subsection{Tasks}
\minibench consists of 20 tasks that are representative of the full 100 tasks comprising \bench (\cref{app:tasks}).

\begin{table}[h]
    \centering
\resizebox{\textwidth}{!}{\input{tables/task_stats/mini}}
    \caption{Tasks included in \minibench.}
\end{table}

\clearpage

\subsection{Evaluation results}
For completeness, we provide the evaluation results on \minibench. The overall ranking of the models, win rates and skill scores align with the scores on the full benchmark reported in \cref{app:extra-results}.

\begin{table*}[h]
    \centering
    \resizebox{\textwidth}{!}{\input{tables/leaderboard_SQL_mini}}
    \caption{Marginal probabilistic forecasting performance of all models (according to the SQL metric) on the \minibench benchmark. The reported metrics are defined in Sections~\ref{sec:aggregation-marginal} and \ref{sec:results}.}
    \label{tab:lb-sql-mini}
\end{table*}

\begin{table*}[h]
    \centering
    \resizebox{\textwidth}{!}{\input{tables/leaderboard_MASE_mini}}
    \caption{Marginal point forecasting performance of all models (according to the MASE metric) on the \minibench benchmark. The reported metrics are defined in Sections~\ref{sec:aggregation-marginal} and \ref{sec:results}.}
    \label{tab:lb-mase-mini}
\end{table*}


\end{document}

%% file: math_commands.tex
\newcommand{\clip}{\operatorname{clip}}
\newcommand{\fev}{\texttt{fev}\xspace}
\newcommand{\bench}{\texttt{fev-bench}\xspace}
\newcommand{\minibench}{\texttt{fev-bench-mini}\xspace}
\crefname{equation}{Eq.}{Eqs.}
\Crefname{equation}{Eq.}{Eqs.}
\crefname{section}{Sec.}{Secs.}
\Crefname{section}{Sec.}{Secs.}

\def\eqref#1{equation~\ref{#1}}

\def\1{\bm{1}}

\def\vy{{\bm{y}}}

\DeclareMathAlphabet{\mathsfit}{\encodingdefault}{\sfdefault}{m}{sl}
\SetMathAlphabet{\mathsfit}{bold}{\encodingdefault}{\sfdefault}{bx}{n}

\def\gQ{{\mathcal{Q}}}

\newcommand{\R}{\mathbb{R}}

%% file: tables/bench_stats.tex
\begin{tabular}{lrrrrl}
\toprule
\textbf{Benchmark} & \textbf{\# datasets} & \textbf{\# tasks} & \textbf{\# tasks with covariates} & \textbf{\# multivariate tasks} & \textbf{Forecast type} \\
\midrule
Monash \citep{godahewa2021monash} & 42 & 42 & 0 & 0 & point \\
LTSF \citep{zeng2023transformers} & 9 & 36 & 0 & 9 & point \\
TFB \citep{qiu2024tfb} & 41 & 116 & 0 & 25 & point \\
BasicTS+ \citep{shao2024exploring} & 20 & 40 & 0 & 20 & point \\
ProbTS \citep{zhang2024probts} & 18 & 18 & 0 & 14 & point \& quantile \\
TSFM-Bench \citep{li2025tsfm} & 21 & 21 & 0 & 21 & point \\
TIME \citep{qiao2026time} & 50 & 50 & 0 & 36 & point \& quantile \\
GIFT-Eval \citep{aksu2024gift} & 55 & 97 & 0 & 43 & point \& quantile \\
\bench (this work) & 96 & 100 & 46 & 35 & point \& quantile \\
\bottomrule
\end{tabular}

%% file: tables/dataset_domains.tex
\begin{tabular}{lrrrrrrr}
\toprule
\textbf{Benchmark} & \textbf{Energy} & \textbf{Nature} & \textbf{Cloud} & \textbf{Mobility} & \textbf{Econ} & \textbf{Health} & \textbf{Retail} \\
\midrule
TIME & 4 & 16 & 3 & 4 & 14 & 3 & 6 \\
GIFT-Eval & 16 & 9 & 8 & 7 & 6 & 5 & 4 \\
\bench & 26 & 5 & 20 & 7 & 10 & 8 & 20 \\
\bottomrule
\end{tabular}

%% file: tables/dataset_freqs.tex
\begin{tabular}{lrrrrrrrrrrrr}
\toprule
\textbf{Benchmark} & \textbf{10S} & \textbf{T} & \textbf{5T} & \textbf{10T} & \textbf{15T} & \textbf{30T} & \textbf{H} & \textbf{D} & \textbf{W} & \textbf{M} & \textbf{Q} & \textbf{Y} \\
\midrule
TIME & 0 & 0 & 7 & 1 & 7 & 1 & 8 & 9 & 4 & 9 & 2 & 0 \\
GIFT-Eval & 2 & 0 & 4 & 2 & 4 & 0 & 13 & 15 & 8 & 5 & 1 & 1 \\
\bench & 0 & 6 & 7 & 2 & 5 & 4 & 22 & 19 & 16 & 7 & 4 & 4 \\
\bottomrule
\end{tabular}

%% file: tables/leaderboard_SQL_subset.tex
\begin{tabular}{lrrrrr}
\toprule
\textbf{Model} & \textbf{Avg. win rate (\%)} & \textbf{Skill score (\%)} & \textbf{Median runtime / 100 series (s)} & \textbf{Leakage (\%)} & \textbf{\# failures} \\
\midrule
Chronos-2 & 91.4 & 47.3 & 0.8 & 0 & 0 \\
TiRex & 83.3 & 42.6 & 0.2 & 1 & 0 \\
TimesFM-2.5 & 79.6 & 42.2 & 1.9 & 10 & 0 \\
FlowState & 73.3 & 39.9 & 2.3 & 8 & 0 \\
Toto-1.0 & 72.4 & 40.8 & 22.3 & 8 & 0 \\
TabPFN-TS & 69.7 & 41.5 & 109.4 & 0 & 1 \\
Moirai-2.0 & 67.4 & 39.3 & 0.4 & 28 & 0 \\
Chronos-Bolt & 67.0 & 38.9 & 0.3 & 0 & 0 \\
TFT & 51.4 & 32.2 & 828.8 & 0 & 0 \\
Stat. Ensemble & 45.5 & 20.7 & 146.9 & 0 & 4 \\
CatBoost & 34.4 & 23.0 & 29.5 & 0 & 0 \\
Seasonal Naive & 19.6 & 0.0 & 0.5 & 0 & 0 \\
\bottomrule
\end{tabular}

%% file: tables/leaderboard_MASE_subset.tex
\begin{tabular}{lrrrrr}
\toprule
\textbf{Model} & \textbf{Avg. win rate (\%)} & \textbf{Skill score (\%)} & \textbf{Median runtime / 100 series (s)} & \textbf{Leakage (\%)} & \textbf{\# failures} \\
\midrule
Chronos-2 & 87.2 & 35.5 & 0.8 & 0 & 0 \\
TiRex & 77.0 & 30.0 & 0.2 & 1 & 0 \\
TimesFM-2.5 & 76.1 & 30.2 & 1.9 & 10 & 0 \\
FlowState & 70.5 & 27.8 & 2.3 & 8 & 0 \\
Toto-1.0 & 66.6 & 28.2 & 22.3 & 8 & 0 \\
TabPFN-TS & 63.9 & 29.8 & 109.4 & 0 & 1 \\
Moirai-2.0 & 62.4 & 27.2 & 0.4 & 28 & 0 \\
Chronos-Bolt & 61.3 & 26.5 & 0.3 & 0 & 0 \\
CatBoost & 52.4 & 23.7 & 29.5 & 0 & 0 \\
Stat. Ensemble & 47.7 & 16.4 & 146.9 & 0 & 4 \\
TFT & 45.7 & 20.5 & 828.8 & 0 & 0 \\
Seasonal Naive & 20.0 & 0.0 & 0.5 & 0 & 0 \\
\bottomrule
\end{tabular}

%% file: tables/task_stats/gift.tex
\begin{tabular}{lllrrrrrrrr}
\textbf{Task} & \textbf{Domain} & \textbf{Freq.} & $H$ & $W$ & \textbf{Median length} & \textbf{\# series} & \textbf{\# targets} & \textbf{\# past cov.} & \textbf{\# known cov.} & \textbf{\# static cov.} \\
\midrule
BizITObs - L2C & cloud & 5T & 288 & 20 & 31,968 & 1 & 7 & 0 & 0 & 0 \\
BizITObs - L2C & cloud & H & 24 & 20 & 2,664 & 1 & 7 & 0 & 0 & 0 \\
ETT & energy & 15T & 96 & 20 & 69,680 & 2 & 7 & 0 & 0 & 0 \\
ETT & energy & H & 168 & 20 & 17,420 & 2 & 7 & 0 & 0 & 0 \\
ETT & energy & D & 28 & 20 & 724 & 2 & 7 & 0 & 0 & 0 \\
ETT & energy & W & 13 & 5 & 103 & 2 & 7 & 0 & 0 & 0 \\
Hierarchical Sales & retail & D & 28 & 10 & 1,825 & 118 & 1 & 0 & 0 & 0 \\
Hierarchical Sales & retail & W & 13 & 10 & 260 & 118 & 1 & 0 & 0 & 0 \\
Hospital & healthcare & M & 12 & 4 & 84 & 767 & 1 & 0 & 0 & 0 \\
Jena Weather & nature & 10T & 144 & 20 & 52,704 & 1 & 21 & 0 & 0 & 0 \\
Jena Weather & nature & D & 28 & 11 & 366 & 1 & 21 & 0 & 0 & 0 \\
Jena Weather & nature & H & 24 & 20 & 8,784 & 1 & 21 & 0 & 0 & 0 \\
Loop Seattle & mobility & D & 28 & 10 & 365 & 323 & 1 & 0 & 0 & 0 \\
Loop Seattle & mobility & 5T & 288 & 10 & 105,120 & 323 & 1 & 0 & 0 & 0 \\
Loop Seattle & mobility & H & 168 & 10 & 8,760 & 323 & 1 & 0 & 0 & 0 \\
M-DENSE & mobility & D & 28 & 10 & 730 & 30 & 1 & 0 & 0 & 0 \\
M-DENSE & mobility & H & 168 & 10 & 17,520 & 30 & 1 & 0 & 0 & 0 \\
SZ Taxi & mobility & 15T & 96 & 10 & 2,976 & 156 & 1 & 0 & 0 & 0 \\
SZ Taxi & mobility & H & 168 & 2 & 744 & 156 & 1 & 0 & 0 & 0 \\
Solar & energy & W & 13 & 1 & 52 & 137 & 1 & 0 & 0 & 0 \\
Solar & energy & D & 28 & 10 & 365 & 137 & 1 & 0 & 0 & 0 \\
\end{tabular}

%% file: tables/task_stats/macro.tex
\begin{tabular}{lllrrrrrrrr}
\textbf{Task} & \textbf{Domain} & \textbf{Freq.} & $H$ & $W$ & \textbf{Median length} & \textbf{\# series} & \textbf{\# targets} & \textbf{\# past cov.} & \textbf{\# known cov.} & \textbf{\# static cov.} \\
\midrule
Australian Tourism & econ & Q & 8 & 2 & 36 & 89 & 1 & 0 & 0 & 0 \\
FRED-MD - CEE & econ & M & 12 & 20 & 798 & 1 & 3 & 4 & 0 & 0 \\
FRED-MD - Macro & econ & M & 12 & 20 & 798 & 1 & 51 & 0 & 0 & 0 \\
FRED-QD - CEE & econ & Q & 8 & 20 & 266 & 1 & 3 & 4 & 0 & 0 \\
FRED-QD - Macro & econ & Q & 8 & 20 & 266 & 1 & 51 & 0 & 0 & 0 \\
GVAR & econ & Q & 8 & 10 & 178 & 33 & 6 & 3 & 0 & 0 \\
US Consumption & econ & M & 12 & 10 & 792 & 31 & 1 & 0 & 0 & 0 \\
US Consumption & econ & Q & 8 & 10 & 262 & 31 & 1 & 0 & 0 & 0 \\
US Consumption & econ & Y & 5 & 10 & 64 & 31 & 1 & 0 & 0 & 0 \\
World CO2 Emissions & econ & Y & 5 & 9 & 60 & 191 & 1 & 0 & 0 & 0 \\
World Life Expectancy & econ & Y & 5 & 10 & 74 & 237 & 1 & 0 & 0 & 0 \\
World Tourism & econ & Y & 5 & 2 & 21 & 178 & 1 & 0 & 0 & 0 \\
\end{tabular}

%% file: tables/task_stats/energy.tex
\begin{tabular}{lllrrrrrrrr}
\textbf{Task} & \textbf{Domain} & \textbf{Freq.} & $H$ & $W$ & \textbf{Median length} & \textbf{\# series} & \textbf{\# targets} & \textbf{\# past cov.} & \textbf{\# known cov.} & \textbf{\# static cov.} \\
\midrule
ENTSO-e Load & energy & 15T & 96 & 20 & 175,292 & 6 & 1 & 0 & 3 & 0 \\
ENTSO-e Load & energy & 30T & 96 & 20 & 87,645 & 6 & 1 & 0 & 3 & 0 \\
ENTSO-e Load & energy & H & 168 & 20 & 43,822 & 6 & 1 & 0 & 3 & 0 \\
EPF-BE & energy & H & 24 & 20 & 52,416 & 1 & 1 & 0 & 2 & 0 \\
EPF-DE & energy & H & 24 & 20 & 52,416 & 1 & 1 & 0 & 2 & 0 \\
EPF-FR & energy & H & 24 & 20 & 52,416 & 1 & 1 & 0 & 2 & 0 \\
EPF-NP & energy & H & 24 & 20 & 52,416 & 1 & 1 & 0 & 2 & 0 \\
EPF-PJM & energy & H & 24 & 20 & 52,416 & 1 & 1 & 0 & 2 & 0 \\
ERCOT & energy & D & 28 & 20 & 6,452 & 8 & 1 & 0 & 0 & 0 \\
ERCOT & energy & H & 168 & 20 & 154,872 & 8 & 1 & 0 & 0 & 0 \\
ERCOT & energy & M & 12 & 15 & 211 & 8 & 1 & 0 & 0 & 0 \\
ERCOT & energy & W & 13 & 20 & 921 & 8 & 1 & 0 & 0 & 0 \\
GFC12 & energy & H & 168 & 10 & 39,414 & 11 & 1 & 0 & 1 & 0 \\
GFC14 & energy & H & 168 & 20 & 17,520 & 1 & 1 & 0 & 1 & 0 \\
GFC17 & energy & H & 168 & 20 & 17,544 & 8 & 1 & 0 & 1 & 0 \\
Solar with Weather & energy & 15T & 96 & 20 & 198,600 & 1 & 1 & 2 & 7 & 0 \\
Solar with Weather & energy & H & 24 & 20 & 49,648 & 1 & 1 & 2 & 7 & 0 \\
\end{tabular}

%% file: tables/task_stats/boomlet.tex
\begin{tabular}{lllrrrrrrrr}
\textbf{Task} & \textbf{Domain} & \textbf{Freq.} & $H$ & $W$ & \textbf{Median length} & \textbf{\# series} & \textbf{\# targets} & \textbf{\# past cov.} & \textbf{\# known cov.} & \textbf{\# static cov.} \\
\midrule
BOOMLET - 1062 & cloud & 5T & 288 & 20 & 16,384 & 1 & 21 & 0 & 0 & 0 \\
BOOMLET - 1209 & cloud & 5T & 288 & 20 & 16,384 & 1 & 53 & 0 & 0 & 0 \\
BOOMLET - 1225 & cloud & T & 60 & 20 & 16,384 & 1 & 49 & 0 & 0 & 0 \\
BOOMLET - 1230 & cloud & 5T & 288 & 20 & 16,384 & 1 & 23 & 0 & 0 & 0 \\
BOOMLET - 1282 & cloud & T & 60 & 20 & 16,384 & 1 & 35 & 0 & 0 & 0 \\
BOOMLET - 1487 & cloud & 5T & 288 & 20 & 16,384 & 1 & 54 & 0 & 0 & 0 \\
BOOMLET - 1631 & cloud & 30T & 96 & 20 & 10,463 & 1 & 40 & 0 & 0 & 0 \\
BOOMLET - 1676 & cloud & 30T & 96 & 20 & 10,463 & 1 & 100 & 0 & 0 & 0 \\
BOOMLET - 1855 & cloud & H & 24 & 20 & 5,231 & 1 & 52 & 0 & 0 & 0 \\
BOOMLET - 1975 & cloud & H & 24 & 20 & 5,231 & 1 & 75 & 0 & 0 & 0 \\
BOOMLET - 2187 & cloud & H & 24 & 20 & 5,231 & 1 & 100 & 0 & 0 & 0 \\
BOOMLET - 285 & cloud & T & 60 & 20 & 16,384 & 1 & 75 & 0 & 0 & 0 \\
BOOMLET - 619 & cloud & T & 60 & 20 & 16,384 & 1 & 52 & 0 & 0 & 0 \\
BOOMLET - 772 & cloud & T & 60 & 20 & 16,384 & 1 & 67 & 0 & 0 & 0 \\
BOOMLET - 963 & cloud & T & 60 & 20 & 16,384 & 1 & 28 & 0 & 0 & 0 \\
\end{tabular}

%% file: tables/task_stats/competitions.tex
\begin{tabular}{lllrrrrrrrr}
\textbf{Task} & \textbf{Domain} & \textbf{Freq.} & $H$ & $W$ & \textbf{Median length} & \textbf{\# series} & \textbf{\# targets} & \textbf{\# past cov.} & \textbf{\# known cov.} & \textbf{\# static cov.} \\
\midrule
Favorita Store Sales & retail & M & 12 & 2 & 54 & 1,579 & 1 & 1 & 1 & 6 \\
Favorita Store Sales & retail & W & 13 & 10 & 240 & 1,579 & 1 & 1 & 1 & 6 \\
Favorita Store Sales & retail & D & 28 & 10 & 1,688 & 1,579 & 1 & 1 & 2 & 6 \\
Favorita Transactions & retail & M & 12 & 2 & 54 & 51 & 1 & 1 & 0 & 5 \\
Favorita Transactions & retail & W & 13 & 10 & 240 & 51 & 1 & 1 & 0 & 5 \\
Favorita Transactions & retail & D & 28 & 10 & 1,688 & 51 & 1 & 1 & 1 & 5 \\
KDD Cup 2022 & energy & D & 14 & 10 & 243 & 134 & 1 & 9 & 0 & 0 \\
KDD Cup 2022 & energy & 10T & 288 & 10 & 35,279 & 134 & 1 & 9 & 0 & 0 \\
KDD Cup 2022 & energy & 30T & 96 & 10 & 11,758 & 134 & 1 & 9 & 0 & 0 \\
M5 & retail & M & 12 & 1 & 58 & 30,490 & 1 & 0 & 8 & 5 \\
M5 & retail & W & 13 & 1 & 257 & 30,490 & 1 & 0 & 8 & 5 \\
M5 & retail & D & 28 & 1 & 1,810 & 30,490 & 1 & 0 & 8 & 5 \\
Restaurant & retail & D & 28 & 8 & 296 & 817 & 1 & 0 & 0 & 4 \\
Rohlik Orders & retail & W & 8 & 5 & 170 & 7 & 1 & 9 & 4 & 0 \\
Rohlik Orders & retail & D & 61 & 5 & 1,197 & 7 & 1 & 9 & 4 & 0 \\
Rohlik Sales & retail & W & 8 & 1 & 150 & 5,243 & 1 & 1 & 13 & 7 \\
Rohlik Sales & retail & D & 14 & 1 & 1,046 & 5,390 & 1 & 1 & 13 & 7 \\
Rossmann & retail & W & 13 & 8 & 133 & 1,115 & 1 & 1 & 4 & 10 \\
Rossmann & retail & D & 48 & 10 & 942 & 1,115 & 1 & 1 & 5 & 10 \\
Walmart & retail & W & 39 & 1 & 143 & 2,936 & 1 & 0 & 10 & 4 \\
\end{tabular}

%% file: tables/task_stats/other.tex
\begin{tabular}{lllrrrrrrrr}
\textbf{Task} & \textbf{Domain} & \textbf{Freq.} & $H$ & $W$ & \textbf{Median length} & \textbf{\# series} & \textbf{\# targets} & \textbf{\# past cov.} & \textbf{\# known cov.} & \textbf{\# static cov.} \\
\midrule
ECDC ILI & healthcare & W & 13 & 10 & 201 & 25 & 1 & 0 & 0 & 0 \\
Hermes & retail & W & 52 & 1 & 261 & 10,000 & 1 & 0 & 1 & 2 \\
Hospital Admissions & healthcare & D & 28 & 20 & 1,731 & 8 & 1 & 0 & 0 & 0 \\
Hospital Admissions & healthcare & W & 13 & 16 & 246 & 8 & 1 & 0 & 0 & 0 \\
Redset & cloud & 5T & 288 & 10 & 25,920 & 118 & 1 & 0 & 0 & 1 \\
Redset & cloud & 15T & 96 & 10 & 8,640 & 126 & 1 & 0 & 0 & 1 \\
Redset & cloud & H & 24 & 10 & 2,160 & 138 & 1 & 0 & 0 & 1 \\
UCI Air Quality & nature & H & 168 & 20 & 9,357 & 1 & 4 & 0 & 3 & 0 \\
UCI Air Quality & nature & D & 28 & 11 & 389 & 1 & 4 & 0 & 3 & 0 \\
UK COVID - Nation - Cumulative & healthcare & D & 28 & 20 & 729 & 4 & 3 & 5 & 0 & 0 \\
UK COVID - Nation - Cumulative & healthcare & W & 8 & 4 & 105 & 4 & 3 & 5 & 0 & 0 \\
UK COVID - Nation - New & healthcare & D & 28 & 20 & 729 & 4 & 3 & 5 & 0 & 0 \\
UK COVID - Nation - New & healthcare & W & 8 & 4 & 105 & 4 & 3 & 5 & 0 & 0 \\
UK COVID - UTLA - Cumulative & healthcare & W & 13 & 5 & 104 & 214 & 1 & 0 & 0 & 0 \\
UK COVID - UTLA - New & healthcare & D & 28 & 10 & 721 & 214 & 1 & 0 & 0 & 0 \\
\end{tabular}

%% file: tables/model_properties.tex
\begin{tabular}{llrrrlr}
\toprule
Model Name & Hugging Face ID & Batch size & Version & Max Context & Hyperparameters \\
\midrule
Chronos-2 & \href{https://huggingface.co/amazon/chronos-2}{\texttt{amazon/chronos-2}} & 100 & 2.2.2 & 8192 & \texttt{\{cross\_learning: True\}} \\
TiRex & \href{https://huggingface.co/NX-AI/TiRex}{\texttt{NX-AI/TiRex}} & 512 & 2025-09-01 & 2048 & - \\
Toto & \href{https://huggingface.co/Datadog/Toto-Open-Base-1.0}{\texttt{Datadog/Toto-Open-Base-1.0}} & 24 & 2025-08-01 & 4096 & \texttt{\{samples\_per\_batch: 8\}} \\
Moirai 2.0 & \href{https://huggingface.co/Salesforce/moirai-2.0-R-small}{\texttt{Salesforce/moirai-2.0-R-small}} & 128 & 2025-08-10 & 4000 & - \\
Chronos-Bolt & \href{https://huggingface.co/amazon/chronos-bolt-base}{\texttt{amazon/chronos-bolt-base}} & 256 & 1.5.3 & 2048 & - \\
TimesFM 2.5 & \href{https://huggingface.co/google/timesfm-2.5-200m-pytorch}{\texttt{google/timesfm-2.5-200m-pytorch}} & 256 & 2025-09-28 & 16000 & - \\
TabPFN-TS & \href{https://huggingface.co/Prior-Labs/TabPFN-v2-reg}{\texttt{Prior-Labs/TabPFN-v2-reg}} & - & 1.0.10 & 4096 & \texttt{\{checkpoint: '2noar4o2'\}} \\
FlowState & \href{https://huggingface.co/ibm-granite/granite-timeseries-flowstate-r1}{\texttt{ibm-granite/granite-timeseries-flowstate-r1}} & 16 & 2026-05-06 & auto & - \\
Sundial & \href{https://huggingface.co/thuml/sundial-base-128m}{\texttt{thuml/sundial-base-128m}} & 512 & 2025-09-01 & 2880 & - \\
\bottomrule
\end{tabular}

%% file: tables/leaderboard_SQL.tex
\begin{tabular}{lrrrrr}
\toprule
\textbf{Model} & \textbf{Avg. win rate (\%)} & \textbf{Skill score (\%)} & \textbf{Median runtime / 100 series (s)} & \textbf{Leakage (\%)} & \textbf{\# failures} \\
\midrule
Chronos-2 & 91.4 & 47.3 & 0.8 & 0 & 0 \\
TiRex & 83.3 & 42.6 & 0.2 & 1 & 0 \\
TimesFM-2.5 & 79.6 & 42.2 & 1.9 & 10 & 0 \\
FlowState & 73.3 & 39.9 & 2.3 & 8 & 0 \\
Toto-1.0 & 72.4 & 40.8 & 22.3 & 8 & 0 \\
TabPFN-TS & 69.7 & 41.5 & 109.4 & 0 & 1 \\
Moirai-2.0 & 67.4 & 39.3 & 0.4 & 28 & 0 \\
Chronos-Bolt & 67.0 & 38.9 & 0.3 & 0 & 0 \\
TFT & 51.4 & 32.2 & 828.8 & 0 & 0 \\
Sundial-Base & 48.0 & 33.4 & 8.0 & 1 & 0 \\
PatchTST & 45.9 & 30.1 & 663.6 & 0 & 0 \\
Stat. Ensemble & 45.5 & 20.7 & 146.9 & 0 & 4 \\
DeepAR & 44.2 & 29.1 & 1113.5 & 0 & 3 \\
AutoARIMA & 41.4 & 21.3 & 20.1 & 0 & 4 \\
CatBoost & 34.4 & 23.0 & 29.5 & 0 & 0 \\
AutoETS & 33.0 & -26.8 & 3.5 & 0 & 3 \\
LightGBM & 31.8 & 21.0 & 3.0 & 0 & 0 \\
AutoTheta & 27.9 & 5.5 & 3.3 & 0 & 0 \\
Seasonal Naive & 19.6 & 0.0 & 0.5 & 0 & 0 \\
Naive & 12.7 & -45.4 & 0.5 & 0 & 0 \\
Drift & 10.1 & -45.8 & 0.5 & 0 & 0 \\
\bottomrule
\end{tabular}

%% file: tables/leaderboard_MASE.tex
\begin{tabular}{lrrrrr}
\toprule
\textbf{Model} & \textbf{Avg. win rate (\%)} & \textbf{Skill score (\%)} & \textbf{Median runtime / 100 series (s)} & \textbf{Leakage (\%)} & \textbf{\# failures} \\
\midrule
Chronos-2 & 87.2 & 35.5 & 0.8 & 0 & 0 \\
TiRex & 77.0 & 30.0 & 0.2 & 1 & 0 \\
TimesFM-2.5 & 76.1 & 30.2 & 1.9 & 10 & 0 \\
FlowState & 70.5 & 27.8 & 2.3 & 8 & 0 \\
Toto-1.0 & 66.6 & 28.2 & 22.3 & 8 & 0 \\
TabPFN-TS & 63.9 & 29.8 & 109.4 & 0 & 1 \\
Moirai-2.0 & 62.4 & 27.2 & 0.4 & 28 & 0 \\
Chronos-Bolt & 61.3 & 26.5 & 0.3 & 0 & 0 \\
Sundial-Base & 53.6 & 24.7 & 8.0 & 1 & 0 \\
CatBoost & 52.4 & 23.7 & 29.5 & 0 & 0 \\
LightGBM & 49.4 & 21.7 & 3.0 & 0 & 0 \\
Stat. Ensemble & 47.7 & 16.4 & 146.9 & 0 & 4 \\
TFT & 45.7 & 20.5 & 828.8 & 0 & 0 \\
PatchTST & 41.3 & 18.4 & 663.6 & 0 & 0 \\
DeepAR & 38.6 & 17.5 & 1113.5 & 0 & 3 \\
AutoARIMA & 36.2 & 11.6 & 20.1 & 0 & 4 \\
AutoTheta & 33.3 & 11.0 & 3.3 & 0 & 0 \\
AutoETS & 33.1 & 2.3 & 3.5 & 0 & 3 \\
Seasonal Naive & 20.0 & 0.0 & 0.5 & 0 & 0 \\
Naive & 18.2 & -16.7 & 0.5 & 0 & 0 \\
Drift & 15.4 & -18.1 & 0.5 & 0 & 0 \\
\bottomrule
\end{tabular}

%% file: tables/task_stats/mini.tex
\begin{tabular}{lllrrrrrrrr}
\textbf{Task} & \textbf{Domain} & \textbf{Freq.} & $H$ & $W$ & \textbf{Median length} & \textbf{\# series} & \textbf{\# targets} & \textbf{\# past cov.} & \textbf{\# known cov.} & \textbf{\# static cov.} \\
\midrule
BOOMLET - 1282 & cloud & T & 60 & 20 & 16,384 & 1 & 35 & 0 & 0 & 0 \\
BOOMLET - 1676 & cloud & 30T & 96 & 20 & 10,463 & 1 & 100 & 0 & 0 & 0 \\
BOOMLET - 619 & cloud & T & 60 & 20 & 16,384 & 1 & 52 & 0 & 0 & 0 \\
BizITObs - L2C & cloud & 5T & 288 & 20 & 31,968 & 1 & 7 & 0 & 0 & 0 \\
EPF-NP & energy & H & 24 & 20 & 52,416 & 1 & 1 & 0 & 2 & 0 \\
ETT & energy & 15T & 96 & 20 & 69,680 & 2 & 7 & 0 & 0 & 0 \\
ETT & energy & H & 168 & 20 & 17,420 & 2 & 7 & 0 & 0 & 0 \\
GFC14 & energy & H & 168 & 20 & 17,520 & 1 & 1 & 0 & 1 & 0 \\
Hospital Admissions & healthcare & W & 13 & 16 & 246 & 8 & 1 & 0 & 0 & 0 \\
Hospital Admissions & healthcare & D & 28 & 20 & 1,731 & 8 & 1 & 0 & 0 & 0 \\
Jena Weather & nature & H & 24 & 20 & 8,784 & 1 & 21 & 0 & 0 & 0 \\
M-DENSE & mobility & D & 28 & 10 & 730 & 30 & 1 & 0 & 0 & 0 \\
Rohlik Orders & retail & D & 61 & 5 & 1,197 & 7 & 1 & 9 & 4 & 0 \\
Rossmann & retail & W & 13 & 8 & 133 & 1,115 & 1 & 1 & 4 & 10 \\
Rossmann & retail & D & 48 & 10 & 942 & 1,115 & 1 & 1 & 5 & 10 \\
Solar with Weather & energy & H & 24 & 20 & 49,648 & 1 & 1 & 2 & 7 & 0 \\
UCI Air Quality & nature & H & 168 & 20 & 9,357 & 1 & 4 & 0 & 3 & 0 \\
UK COVID - Nation - Cumulative & healthcare & D & 28 & 20 & 729 & 4 & 3 & 5 & 0 & 0 \\
US Consumption & econ & Y & 5 & 10 & 64 & 31 & 1 & 0 & 0 & 0 \\
World CO2 Emissions & econ & Y & 5 & 9 & 60 & 191 & 1 & 0 & 0 & 0 \\
\end{tabular}

%% file: tables/leaderboard_SQL_mini.tex
\begin{tabular}{lrrrrr}
\toprule
\textbf{Model} & \textbf{Avg. win rate (\%)} & \textbf{Skill score (\%)} & \textbf{Median runtime / 100 series (s)} & \textbf{Leakage (\%)} & \textbf{\# failures} \\
\midrule
Chronos-2 & 90.3 & 51.1 & 1.2 & 0 & 0 \\
TiRex & 79.5 & 43.4 & 0.4 & 0 & 0 \\
TimesFM-2.5 & 76.2 & 44.0 & 5.2 & 5 & 0 \\
Toto-1.0 & 73.2 & 43.0 & 44.1 & 5 & 0 \\
TabPFN-TS & 70.2 & 47.6 & 270.3 & 0 & 0 \\
FlowState & 70.2 & 41.8 & 3.2 & 5 & 0 \\
Moirai-2.0 & 66.4 & 41.7 & 0.4 & 30 & 0 \\
Chronos-Bolt & 63.6 & 40.4 & 0.5 & 0 & 0 \\
TFT & 53.9 & 35.3 & 1889.7 & 0 & 0 \\
PatchTST & 51.2 & 35.0 & 1517.7 & 0 & 0 \\
Sundial-Base & 49.2 & 37.6 & 8.3 & 0 & 0 \\
DeepAR & 46.4 & 30.9 & 2863.3 & 0 & 0 \\
AutoARIMA & 46.0 & 31.4 & 95.1 & 0 & 1 \\
Stat. Ensemble & 46.0 & 28.1 & 252.3 & 0 & 1 \\
CatBoost & 39.2 & 28.4 & 95.6 & 0 & 0 \\
LightGBM & 35.5 & 24.0 & 7.7 & 0 & 0 \\
AutoETS & 32.2 & -9.4 & 6.2 & 0 & 0 \\
AutoTheta & 25.0 & 8.5 & 5.1 & 0 & 0 \\
Seasonal Naive & 15.5 & 0.0 & 0.8 & 0 & 0 \\
Naive & 11.8 & -36.9 & 0.8 & 0 & 0 \\
Drift & 8.2 & -36.0 & 0.8 & 0 & 0 \\
\bottomrule
\end{tabular}

%% file: tables/leaderboard_MASE_mini.tex
\begin{tabular}{lrrrrr}
\toprule
\textbf{Model} & \textbf{Avg. win rate (\%)} & \textbf{Skill score (\%)} & \textbf{Median runtime / 100 series (s)} & \textbf{Leakage (\%)} & \textbf{\# failures} \\
\midrule
Chronos-2 & 85.5 & 40.4 & 1.2 & 0 & 0 \\
TimesFM-2.5 & 73.0 & 32.3 & 5.2 & 5 & 0 \\
TiRex & 72.8 & 31.4 & 0.4 & 0 & 0 \\
FlowState & 71.0 & 29.9 & 3.2 & 5 & 0 \\
Toto-1.0 & 68.8 & 31.0 & 44.1 & 5 & 0 \\
TabPFN-TS & 65.2 & 37.3 & 270.3 & 0 & 0 \\
Moirai-2.0 & 62.4 & 30.1 & 0.4 & 30 & 0 \\
Chronos-Bolt & 59.4 & 28.4 & 0.5 & 0 & 0 \\
Sundial-Base & 58.5 & 29.8 & 8.3 & 0 & 0 \\
CatBoost & 56.2 & 30.5 & 95.6 & 0 & 0 \\
LightGBM & 50.7 & 26.2 & 7.7 & 0 & 0 \\
TFT & 50.0 & 25.4 & 1889.7 & 0 & 0 \\
Stat. Ensemble & 48.7 & 22.0 & 252.3 & 0 & 1 \\
PatchTST & 43.0 & 23.9 & 1517.7 & 0 & 0 \\
AutoARIMA & 42.5 & 21.3 & 95.1 & 0 & 1 \\
DeepAR & 37.8 & 20.2 & 2863.3 & 0 & 0 \\
AutoTheta & 32.8 & 14.8 & 5.1 & 0 & 0 \\
AutoETS & 32.0 & 2.3 & 6.2 & 0 & 0 \\
Seasonal Naive & 14.0 & 0.0 & 0.8 & 0 & 0 \\
Naive & 13.2 & -17.4 & 0.8 & 0 & 0 \\
Drift & 12.5 & -17.2 & 0.8 & 0 & 0 \\
\bottomrule
\end{tabular}